\documentclass[twoside,11pt]{article}
 
\usepackage{blindtext}
\usepackage[abbrvbib, preprint]{jmlr2e}

\usepackage{amsmath}
\usepackage{enumitem}
\usepackage{tikz}

\usepackage{lastpage}
\jmlrheading{25}{2024}{1-\pageref{LastPage}}{1/24; Revised TBD}{TBD}{21-0000}{Heinrich Van Deventer and Anna S. Bosman}
\ShortHeadings{Distal Interference Limits Continual Learning}{Van Deventer and Bosman}
\firstpageno{1}

\begin{document}

\title{Distal Interference: Exploring the Limits of Model-Based Continual Learning}

% \title{Distal Interference: Exploring the Limits of Model-Based Continual Learning in Artificial Neural Networks}

\author{\name Heinrich van Deventer \email heinrich.vandeventer@outlook.com \\
       \addr Department of Computer Science\\
       University of Pretoria\\
       Pretoria, Gauteng, South Africa
       \AND
       \name Anna S. Bosman \email anna.bosman@up.ac.za \\
       \addr Department of Computer Science\\
       University of Pretoria\\
       Pretoria, Gauteng, South Africa }

\editor{My editor}

\maketitle

\begin{abstract}%   <- trailing '%' for backward compatibility of .sty file
%\blindtext
Continual learning is the sequential learning of different tasks by a machine learning model. Continual learning is known to be hindered by catastrophic interference or forgetting, which is the rapid unlearning of previously learned tasks when new tasks are learned. Artificial neural networks (ANNs) are prone to catastrophic interference despite their practical success. This study analyses how gradient descent and overlapping or non-orthogonal representations between distant input points lead to distal interference and catastrophic interference. Distal interference refers to the phenomenon where training a model on a subset of the domain leads to non-local changes on other subsets of the domain. This study shows that uniformly trainable models without distal interference must be exponentially large. A novel antisymmetric bounded exponential layer B-spline ANN architecture named \textit{ABEL-Spline} is proposed that can approximate any continuous function, is uniformly trainable, has polynomial computational complexity, and provides \textit{some} guarantees for mitigating distal interference. Experiments are presented to demonstrate the theoretical properties of ABEL-Splines. It is concluded that the weaker distal interference guarantees in ABEL-Splines are insufficient for model-only continual learning. It is conjectured that continual learning with polynomial complexity models requires augmentation of the training data or algorithm. 

\end{abstract}


\section{Introduction}

Continual learning is the process of learning to perform multiple tasks sequentially without forgetting previously learned tasks. Humans can learn many tasks sequentially, unlike artificial neural networks (ANNs) trained with gradient descent optimisation~\citep{McCloskey1989}. The challenge of continual (incremental or life-long) learning is exacerbated by the unavailability of data from old tasks while training on new tasks~\citep{PARISI201954, chen2018lifelong, Delange_2021, vandeVen2022}. Catastrophic interference or forgetting hinders continual learning in ANNs~\citep{french1999catastrophic, kemker2018measuring, robins1995catastrophic}. The catastrophic interference phenomenon occurs when learning a new task interferes with the performance of previously learned tasks by modifying trainable parameters. If an ANN cannot effectively learn many tasks, it has limited utility for continual learning~\citep{hadsell2020embracing, kaushik2021understanding}. Catastrophic interference is like learning to pick up a cup while forgetting how to drink. 

Previous studies on catastrophic interference and mitigation techniques focus on \textbf{time}, or the sequential aspect of the continual learning problem~\citep{PARISI201954, Delange_2021}. This study considers the input \textbf{space} for differentiable models such as ANNs. A differentiable model $f_{t}$ trained on data points $v \in D(f)$ to obtain $f_{t+1}$ with gradient descent optimisation can have non-local or distal changes $|f_{t+1}(x)-f_{t}(x)|>\varepsilon$ at points $x \in D(f)$ far away $d(x,v)>\delta$ from the training data $v$. Non-local or distal changes to the model are akin to off-target effects, referred to in this study as \textit{distal interference}. We propose distal interference as a novel explanatory mechanism for catastrophic interference and potentially other issues related to slow training and convergence of large ANNs. Training ANNs is akin to erecting a circus tent with countless cables -- when you tighten or improve one part, another part loosens.

The concept of a `stability-plasticity' spectrum or trade-off has been used to describe the difference between near-perfect memory models such as lookup tables and the adaptive, easy-to-train models such as ANNs with unstable memory~\citep{PARISI201954, Delange_2021, french1999catastrophic}. This study proposes more precise definitions for mathematical analysis and discusses how these characteristics interact with distal interference. Stable models are immune to the proposed distal interference mechanism and can be called \textit{distal orthogonal} models due to the geometry of their parameter space (see Section~\ref{sec:learning_without_distal_interference}). It has been noted that orthogonal activity patterns prevent interference, but distributed (overlapping or non-orthogonal representations; see Definition~\ref{def:overlapping_representations}) are theorised to promote generalisation~\citep{HintonDistributedRepresentations}. Plasticity is related to trainability, universal expressiveness, and generalisation. \textit{Uniform} trainability, or non-zero parameter gradients for any parameters and input point, eases training with gradient descent optimisation (see Definition~\ref{def:uniform_trainability}, in Section~\ref{sec:preliminaries}). Expressiveness refers to universal function approximation. 

The stability-plasticity spectrum exists due to varying model geometry and computational complexity. This study formally proves that a distance measure $d(x,v)>\delta$ based on the maximum norm $d(x,v) = \max_{i} (|x_{i} - v_{i}|)$ leads to an exponentially large distal orthogonal model capable of continual learning. This study further contributes with a novel uniformly trainable \textit{antisymmetric bounded exponential layer spline} (ABEL-Spline) architecture and proves an accompanying universal function approximation theorem. ABEL-Splines have a linear space complexity using a weaker distal orthogonality dissimilarity measure $d(x,v) = \min_{i} (|x_{i} - v_{i}|)$. The proposed ABEL-Splines exhibit all the intended properties, but they may not be sufficient for practical model-only continual learning without augmented training data or algorithms. There have been many suggestions to improve continual learning outcomes with augmentation techniques such as pseudo-rehearsal and orthogonal gradient descent~\citep{robins1995catastrophic, look_up_tables_are_robust, ogd_paper}.

It is understood that replay occurs with memory consolidation in human brains~\citep{Ji2007}. However, the exact learning mechanisms in the human brain are only partially understood. An analogous machine learning system that is efficiently computable may require a similar form of augmentation, such as pseudo-rehearsal~\citep{robins1995catastrophic}. Training and data augmentation techniques superficially resemble the replay and memory consolidation mechanisms in human brains, which are facilitated by the hippocampus, among other complex mechanisms~\citep{Ji2007, Bliss1993,  Squire1991}. It is an open question whether polynomial complexity models are capable of model-only sequential learning without such augmentation.

%\clearpage
\section{Preliminaries}
\label{sec:preliminaries}

\begin{definition}[model perturbation] 
\label{def:model_perturbation}
Let $f_{t}(x)$ be a differentiable model with $x \in D(f)$ at time step $t$. Suppose $f_{t}(x)$ is updated for one time step to yield $f_{t+1}(x)$. The perturbation between $f_{t}(x)$ and $f_{t+1}(x)$ can be measured with a function norm:
\begin{equation} 
\label{eq_function_norms_absolute_model_perturbation}
%\begin{split}
\textup{model perturbation} :=  \lVert f_{t+1} - f_{t} \rVert_{1}  
  = \int_{D} |f_{t+1}(x) - f_{t}(x)| \,dx \\
%\end{split}
\end{equation}
\end{definition}

\begin{remark}
One can quantify the difference over the domain with a function norm as shown in Equation~(\ref{eq_function_norms_absolute_model_perturbation}). This model-agnostic measure does not explicitly depend on the particular model architecture but only considers the function represented by a differentiable model $f$. 
\end{remark}

\begin{definition}[distal interference] 
\label{def:model_dsital_interference}
Let $f_{t}(x)$ be a differentiable model with $x \in D(f)$ at time step $t$. Suppose $f_{t}(x)$ is updated for one time step with training data $v \in D$ and target values $y \in \mathbb{R}$ to yield $f_{t+1}(x)$. For some chosen distance or dissimilarity measure $d(x,v)$ and $\delta>0$, one can construct a subset of \textbf{distal} points $D_{v}=\{ x \mid  d(x,v)>\delta, \; \forall x,v \in D \}$. The model change over distal points at time step $t$ is defined as \textbf{distal interference}:
\begin{equation} 
\label{eq_distal_interference_definition}
%\begin{split}
\textup{distal interference} :=  \int_{D_{v}} |f_{t+1}(x) - f_{t}(x)|  \,dx \\
%\end{split}
\end{equation}
\end{definition}

Distal interference is akin to off-target effects during training. Changing model parameters to improve performance on training data $v \in D$ can lead to the model output changing on distant and potentially unrelated parts of the domain. This non-local effect is denoted \textit{distal interference}. Distant and unrelated changes are probably detrimental to model performance, but this may not always be true. If one can limit model perturbation and distal interference for one training or update step, then one can determine the change or drift of the model trained for many steps. Many models, such as ANNs, are susceptible to distal and catastrophic interference. Distal interference is simpler and easier to analyse than catastrophic interference since it only depends on the model's geometry, not the sequence of tasks. Catastrophic and distal interference is caused due to \textit{overlapping representations}.
%can remove, not the sequence of tasks.

\begin{definition}[overlapping representation]
\label{def:overlapping_representations}
Two points $x,v \in D(f)$ have overlapping representation in a model $f(x)$ with trainable parameters $\theta$ if: 
$$
  \nabla_{\theta} f(x) \cdot \nabla_{\theta} f(v)  \neq 0 
$$
\end{definition}

\begin{remark}
    Distance is not part of the definition of overlapping representations. Heuristically, points that are close to each other should have overlapping representations. Conversely, distant points should ideally have non-overlapping representations. 
\end{remark}

%\clearpage

%\clearpage
Constructing models that guarantee zero distal interference for a specific choice of distance or dissimilarity measure $d(x,v)$ and a fixed $\delta>0$ is possible. One need only guarantee that the models' parameter gradients are orthogonal (and sparse) for a chosen $d(x,v) > \delta$. This study considers the difference measures $\max_{i} (|x_{i} - v_{i}|)$ and $\min_{i} (|x_{i} - v_{i}|)$. 

\begin{definition}[distal orthogonal model]
    \label{def:distal_orthogonal_model}
    Let f(x) be a differentiable model. Given some fixed $\delta>0$ and non-negative dissimilarity or distance measure $d(x,v)$, a model $f(x)$ is called distal orthogonal w.r.t.  $d(x,v)$ if for any trainable parameters $\theta \in \Theta$ and $\forall x, v \in D(f)$: 
    $$
     d(x,v) > \delta \implies  \nabla_{\theta} f(x) \cdot \nabla_{\theta} f(v) = 0 
    $$
\end{definition}

\begin{definition}[max-distal orthogonal model]
    \label{def:max_distal_orthogonal_model}
    Let f(x) be a differentiable model. Given some $\delta>0$, a model $f(x)$ is max-distal orthogonal if for any trainable parameters $\theta \in \Theta$ and $\forall x, v \in D(f)$: 
    $$
     \max_{i} (|x_{i} - v_{i}|) > \delta \implies  \nabla_{\theta} f(x) \cdot \nabla_{\theta} f(v) = 0 
    $$
\end{definition}

\begin{remark}
    A trivial max-distal orthogonal model has no trainable parameters or a gradient vector that is zero everywhere.
\end{remark}

\begin{definition}[uniform trainability]
\label{def:uniform_trainability}
A model $f(x)$ is \textbf{uniformly trainable} if the parameter gradient of the function w.r.t.  $\theta$ is a non-zero vector: 
\begin{equation*}
    \nabla_{\theta} f(x) \neq \Vec{0}, \; \forall x \in D(f), \; 
    \forall \theta \in \Theta
\end{equation*}
\end{definition}

\begin{remark}
Uniform trainability means that a model can be trained with gradient descent on any input, so there are no dead zones with no gradient, which is possible with ReLU ANNs and the `dying ReLU' problem~\citep{Gao_Cai_Ji_2020, hanin2019universal, HintonRelu2010}. A vanishing but non-zero gradient can also hinder trainability. 
Uniform trainability means model parameters can be adjusted with gradient descent at any point in the domain. A model with \textbf{one} adjustable parameter with the same value everywhere is a trivial uniformly trainable model. More complicated and expressive models can also be uniformly trainable.
\end{remark}

\begin{definition}[lookup table]
    \label{def_lookup_table}
    A lookup table model with partition number $z \in \mathbb{N}$ partitions the domain $[0,1]^{n} \subset \mathbb{R}^{n}$ into $z^{n}$ equally sized hyper-cubes of equal size, and maps any point inside each partition to some trainable value. For any $\delta>0$ there exists a lookup table model $f(x)$ with partition number $z \geq \delta^{-1}$ s.t. $f(x)$ is a max-distal orthogonal model:
    $$
     \max_{i} (|x_{i} - v_{i}|) > \delta \implies  \nabla_{\theta} f(x) \cdot \nabla_{\theta} f(v) = 0 
    $$
\end{definition}

A lookup table can distinguish between points $x,y \subset D(f)$ that differ sufficiently in any \textbf{one} of their coordinates. A lookup table associates independent trainable parameters with no overlapping representation to sufficiently different inputs. Lookup tables are also uniformly trainable models over $[0,1]^{n}$. It remains to be shown if uniformly trainable and max-distal orthogonal models have a particular computational complexity.

\clearpage

\section{Learning Without Distal Interference}
\label{sec:learning_without_distal_interference}
Distal interference occurs when parameter updates to change a model output at a specific point affect the model outputs far from the training data point of interest. This non-local interference is an underlying mechanism that can cause catastrophic interference in a continual learning context. This section shows that models with distal orthogonality can learn without distal interference. The mathematical ideas used in this section are similar to neural tangent kernels by \citet{NEURIPS2018_tangent_kernel}. A model $f(\theta_{t},x)$ with initial parameters $\theta_{t}$ can be trained with gradient descent optimisation to create an updated model $f(\theta_{t+1},x)$. The local linear approximation of a $d$-distal orthogonal model $f(\theta_{t+1},x)$ with explicitly shown parameters $\theta$ is given by: 
\begin{equation*}
 f(\theta_{t+1},x) \approx f(\theta_{t},x) + \nabla_{\theta}f(\theta_{t},x) \cdot (\theta_{t+1} - \theta_{t})   
\end{equation*}

\noindent
Let $ \theta_{t+1} = \theta_{t} - \eta \hat{g}$ be the updated parameters of the model trained at some point $v \in D(f)$, with learning rate $\eta$. Let $\hat{g}=(\partial_{f} L)  \,\nabla_{\theta}f(\theta_{t},v)$, where $\partial_{f} L$ is the derivative of the loss function w.r.t.  the model $f$, and $\nabla_{\theta}f(\theta_{t},v)$ be the parameter gradient of the model evaluated at $v$. Then it follows that:
\begin{equation*} 
\label{eq_local_linear_approx_simplification}
\begin{split}
f(\theta_{t+1},x) 
&  \approx f(\theta_{t},x) + \nabla_{\theta}f(\theta_{t},x) \cdot ((\theta_{t} - \eta \hat{g} - \theta_{t}) \\
&  = f(\theta_{t},x) - \eta \; \nabla_{\theta}f(\theta_{t},x) \cdot  \hat{g}  \\
& = f(\theta_{t},x) - \eta \; (\partial_{f} L) \, \nabla_{\theta}f(\theta_{t},x) \cdot  \nabla_{\theta}f(\theta_{t},v)  \\
\end{split}
\end{equation*}

\noindent
The resulting model perturbation from Definition~\ref{def:model_perturbation} is approximated by
\begin{equation}
\begin{split}
\int_{D} |f_{t+1}(x) - f_{t}(x)| \,dx
& \approx \int_{D} |(f_{t}(x) - \eta \nabla_{\theta}f_{t}(x) \cdot \hat{g}-f_{t}(x)| \,dx  \\
& = \int_{D} \eta |  \nabla_{\theta}f_{t}(x) \cdot \hat{g} | \,dx \\
& = \int_{D} \eta \, | \partial_{f} L| \, |  \nabla_{\theta}f_{t}(x) \cdot   \,\nabla_{\theta}f_{t}(v) | \,dx
\\
\end{split}
\end{equation}

\noindent
Assume that $d(x,v)>\delta$, then for any $\theta$ it follows:
\begin{equation*}
    d(x,v)>\delta \implies  \nabla_{\theta} f(\theta, x) \cdot \nabla_{\theta} f(\theta, v) = 0 
\end{equation*}

\noindent
Since $\nabla_{\theta} f_{t}(x) \cdot \nabla_{\theta} f_{t}(v) = 0$, it follows from Defintion~\ref{def:model_dsital_interference} that the \textbf{distal interference} over a subset of the domain $D_{v}=\{ x \mid  d(x,v)>\delta, \; \forall x,v \in D \}$ is approximated by:
\begin{equation} 
\label{eq:zero_distal_interference_linear_approx}
\int_{D_{v}} |f_{t+1}(x)-f_{t}(x)| \,dx \approx \int_{D_{v}} \eta \, | \partial_{f} L| \, |  \nabla_{\theta}f_{t}(x) \cdot   \,\nabla_{\theta}f_{t}(v) | \,dx = 0 \\
\end{equation}

Thus, a model with distal orthogonality can learn with gradient descent at a point $v$ without affecting the values at a distant point $x$ in the domain. Distal interference can deleteriously affect a model's ability to learn continuously. Distal interference is most troublesome when two distant points have large overlapping representations. The required model size for robustness to distal interference still needs to be determined. Such insight could shed light on the computational complexity bounds for models capable of continual learning. 

\begin{remark}
     Ideally, gradient vectors should be orthogonal and \textbf{sparse}; in other words, the gradient vector should consist of mostly zeroed components. If most components are zero, then sparsity improves robustness with a more accurate approximation in Equation~(\ref{eq:zero_distal_interference_linear_approx}). It is, in general, possible for two vectors to be orthogonal with non-zero components. However, finite precision and optimisers that modify update vectors can violate orthogonality assumptions. 
\end{remark}

\section{Limits of Model-Based Continual Learning}
\label{sec:theoretical_limits_model_continual}
Desirable model properties considered in this study include uniform trainability and zero distal interference, using different distance and dissimilarity measures. A model updated with training data $v \in D$ should readily change and update their output values on points $x \in D$ that are $d(x,v) \leq \delta$ close to the training data. However, models should not change at points $x \in D$ far from $d(x,v) > \delta$ the training data.

\begin{theorem}
    If a model $f(x)$ with trainable parameters $\theta$ is uniformly trainable and max-distal orthogonal, then it has a parameter space of at least $\mathcal{O}(z^{n})$ dimensions.
\end{theorem}

\begin{proof}
Consider a max-distal orthogonal model $f(x)$ with domain $D(f) = [0,1]^{n}$, and choose $\delta > 0$. Let $x,v \in [0,1]^{n}$, by Definition~\ref{def:max_distal_orthogonal_model}:
\begin{equation*}
    \max_{i} (|x_{i} - v_{i}|) > \delta \implies  \nabla_{\theta} f(x) \cdot \nabla_{\theta} f_{t}(v) = 0 
\end{equation*}

Assume the model is uniformly trainable, then by Definition~\ref{def:uniform_trainability}:
\begin{equation*}
\nabla_{\theta}f(x) \neq \vec{0}, 
\forall x \in [0,1]^{n}
 , \; 
\forall \theta \in \Theta
\end{equation*}

Choose a partition number $z \in \mathbb{N}$ s.t. $z \geq \delta^{-1}$. Consider a set of all $\mathcal{O}(z^{n})$ grid-points $u^{(i)}$ such that:
\begin{equation*}
    ||u^{(i)}-u^{(j)}||_{\infty} > \delta \geq z^{-1}, \; \forall \;  i \neq j \in \mathbb{N}
\end{equation*}

From distal orthogonality, it follows that:
\begin{equation*}
    \nabla_{\theta} f(u^{(i)}) \cdot \nabla_{\theta} f(u^{(j)}) = 0, \; \forall \;  i \neq j \in \mathbb{N}
\end{equation*}

\noindent
There are $\mathcal{O}(z^{n})$ non-zero parameter gradient vectors that are orthogonal to each other. Thus, the set of vectors $\nabla_{\theta} f(u^{(i)})$ constitute a basis for a vector space with dimension of $\mathcal{O}(z^{n})$. It follows that the model requires at least $\mathcal{O}(z^{n})$ parameters.
\end{proof}

The logically equivalent contra-positive of the theorem is: If a model does not have an exponentially large parameter space, then it is not uniformly trainable or not a max-distal orthogonal model. This implies that \textbf{polynomial complexity models are not uniformly trainable or not max-distal orthogonal models}. In other words, polynomial complexity models have zero gradients and are untrainable for some inputs, or training on one part of the domain will interfere with learned values on distant points. Due to the equivalence of norms in finite dimensions, one should expect similar computational bounds for any norm-based distance function $d(x,v)$. 

\clearpage
Uniform trainability enables training with gradient descent. Memory retention guaranteed due to max-distal orthogonality can counteract catastrophic interference with sequential training on different tasks. Polynomial or low-complexity models are more practical than exponentially large models. The trade-offs and unsatisfiability are visualised in Figure~\ref{fig:triangle_of_tradeoffs}. The findings align with the fact that lookup tables are robust to catastrophic interference but require exponentially many parameters as a function of the input dimension. 

\begin{figure}[h]
\centering
\scalebox{0.8}{
\begin{tikzpicture}[scale=2.5]
    \draw[thick] (0,0) -- (2,0) -- (1,1.73) -- cycle;
    \node[below] at (0,0) {Easy to Train};
    \node[below, align=center] at (2., 0) {Max-Distal\\Memory Retention};
    \node[align=center,above] at (1,1.73) {Low Complexity\\Models};
\end{tikzpicture}}
\caption[Trade-off or unsatisfiability triangle]{The trade-off triangle between computational complexity, ease of optimisation with uniform trainability, and max-distal orthogonal memory retention.}
\label{fig:triangle_of_tradeoffs}
\end{figure}
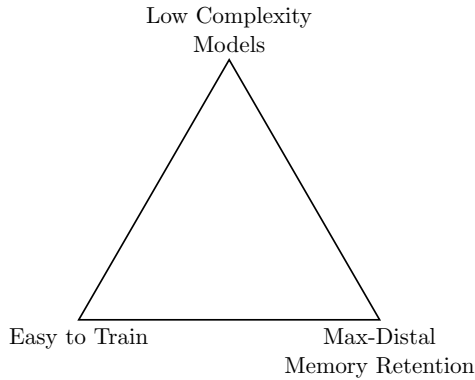

The aim is to find a reasonable trade-off for effective continual learning, which can be visualised in Figure~\ref{fig:triangle_of_tradeoffs_with_aim}. Such an aim may not be achievable, since it is unclear how crucial max-distal orthogonality is for practical memory retention. One could instead opt for some non-negative dissimilarity measure for $d(x,v)$ instead of a proper metric or distance function that adheres to the axioms for a metric space (Appendix~\ref{sec:fundamental_def_and_theorems}, Definition~\ref{def:matric_space}). However, it is not clear whether a dissimilarity measure such as $d(x,v) = \min_{i} (|x_{i} - v_{i}|)$ would be sufficient for continual learning using specifically designed distal orthogonal architectures. 

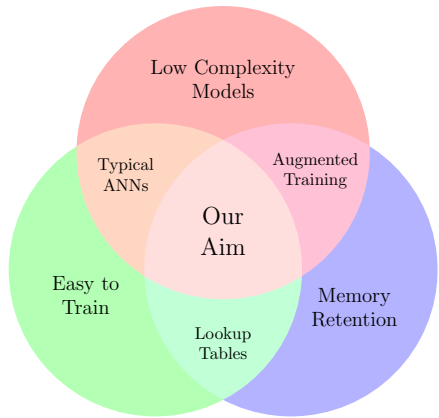
\begin{figure}[h]
\centering
\scalebox{0.65}{
\begin{tikzpicture}
  \begin{scope}[blend group = soft light]
    \fill[red!30!white]   ( 90:1.6) circle (3);
    \fill[green!30!white] (210:1.6) circle (3);
    \fill[blue!30!white]  (330:1.6) circle (3);
  \end{scope}
  \node[align=center] at ( 33:2.25) [font=\small]  {Augmented  \\  Training};
  \node[align=center] at ( 90:3.1) [font=\large]    {Low Complexity\\Models};
  \node[align=center] at ( 150:2.3) [font=\small]   {Typical\\ANNs};
  \node[align=center] at ( 205:3.1) [font=\large]  {Easy to \\ Train};
  \node[align=center] at ( 270:2.3) [font=\small]  {Lookup \\ Tables};
  \node[align=center] at ( 330:3.1) [font=\large]   {Memory \\ Retention};
  \node[align=center] [font=\Large] {Our \\ Aim};
\end{tikzpicture}}
\caption[Venn diagram of key concepts.]{A Venn diagram relating key concepts in machine learning related to this study.}
\label{fig:triangle_of_tradeoffs_with_aim}
\end{figure}

\section{Min-Distal Orthogonal Models}
\label{sec:min_distal_ortho_models}

A different choice of $d(x,v)$ might still enable continual learning without exponential computational complexity. The distance function $d(x,v) = \max_{i} (|x_{i} - v_{i}|)$ is a proper metric or distance function. This study also considers the function $d(x,v) = \min_{i} (|x_{i} - v_{i}|)$ which satisfies \textit{some} of the properties of a metric, but not all the listed properties. Specifically, $\min_{i} (|x_{i} - v_{i}|)$ does not satisfy the triangle inequality nor the identity of indiscernibles. Therefore, $\min_{i} (|x_{i} - v_{i}|)$ is a symmetric and non-negative \textit{dissimilarity} measure.

\begin{definition}[min-distal orthogonal model]
    \label{def:min_distal_orthogonal_model}
    Let f(x) be a differentiable model. Given some $\delta>0$, a model $f(x)$ is min-distal orthogonal if for any trainable parameters $\theta \in \Theta$ and $\forall \; x, v \in D(f)$: 
    $$
     \min_{i} (|x_{i} - v_{i}|) > \delta \implies  \nabla_{\theta} f(x) \cdot \nabla_{\theta} f(v) = 0 
    $$
\end{definition}
This study constructs polynomial complexity models from single-variable cardinal cubic B-splines such that the resulting models exhibit: 
\begin{enumerate}[itemsep=0.1mm,parsep=0.1cm]
    \item \textbf{Sparse activity}---most neural units are zero and inactive.
    \item \textbf{Bounded parameter gradients} regardless of model size.
    \item \textbf{Uniformly trainability} anywhere in the domain.
    \item \textbf{Min-distal orthogonality}.
\end{enumerate} 

\subsection{Cardinal Cubic B-splines}
\label{sec:properties_of_cardinal_splines}

Splines are piece-wise defined polynomial functions, often used in computer graphics, function representation and data interpolation~\citep{de1978practical, prautzsch2013bezier, salomon2007curves}. Splines are akin to stitching together (almost) unrelated polynomials to make a function that is a polynomial locally but not globally over its domain. A common variant of splines that constitute a basis for a vector space (closed under vector operations) are called B-splines~\citep{schoenberg1973cardinal, hollig2012finite, hollig2015approximation}. 

B-splines constitute a flexible and expressive method for function approximation. In addition, analysing B-splines and characterising their general behaviour is simpler than other basis functions. The support interval or domain over which a spline is defined is partitioned into sub-intervals with boundaries called knots. The number of subintervals dictates the number of B-spline basis functions. The degree or order of a B-spline (such as zeroth or cubic) is the degree of each locally defined polynomial of the spline~\citep{schoenberg1973cardinal, hollig2012finite, hollig2015approximation, goldman1993knot}.

The use of B-splines was inspired by hierarchical temporal memory (HTM) systems that use sparse distributed representations (SDRs) to encode information~\citep{hawkins2021thousand, Cui_2016, George_2009}. SDRs are non-linear mappings of scalars to sparse arrays of mostly zeroes, often used as a pre-processing step for models. The choice of B-splines followed from a revelation that SDRs with linear operations are equivalent to \textit{zeroth-order} B-splines, but it is not mentioned in the relevant literature~\citep{HintonDistributedRepresentations}. Zeroth order B-splines resemble single-variable lookup tables. It is also known that lookup tables are very robust to catastrophic interference \citep{look_up_tables_are_robust}. Cubic B-splines can be thought of as smoothed SDRs that may provide the basis for more robust memory in models.

\clearpage
\textit{Cardinal} cubic B-splines are defined on equally sized sub-intervals or partitions. The boundaries of the partitions are called knots, and knot-insertion algorithms can be used to re-partition into smaller sub-intervals~\citep{goldman1993knot}. Splines have been employed in ANNs~\citep{lane1991multi, douzette2017b, Scardapane_2018}, but not for the purpose of memory retention. This study aims to improve memory retention in ANNs with spline-based architectures. The number of basis functions or partitions is fixed and chosen independently of the training data. The spline models considered in this study do not interpolate between consecutive data points. Exact interpolation of training data is ill-advised for tasks with noisy training data. The resulting model would have severe variance and oscillate wildly because of overfitting or interpolating with zero error between each consecutive noisy data point on some interval. The spline partitions in this study are also evenly spaced (cardinal) for simplicity. 
Any continuous single-variable function can be approximated with a cardinal cubic B-spline $f(x)$ with basis functions $S_{i}(x)$, computed with a custom activation function $S(x)$ given in Equation~(\ref{eq:cubic_spline_activation}).
\begin{definition}[$z$-density B-spline function]
\label{def:lambda_desnity_b_spline}
A $z$-density B-spline function, $f$, is a cardinal cubic B-spline function of one variable with $4z+3$ basis functions, partition number $z \in \mathbb{N}$ and adjustable parameters $\theta_{i}$:
\begin{equation*} %\label{eq1}
\begin{split}
f(x) 
= \sum_{i=1}^{4z+3} \theta_{i} S_{i}(x) 
= \sum_{i=1}^{4z+3} \theta_{i} S(w_{i}x + b_{i})
= \sum_{i=1}^{4z+3} \theta_{i} S( (4z) x + (4 - i) )
\end{split}
\end{equation*}
\end{definition}

%\clearpage
\noindent
Cardinal B-splines uniformly partition the interval $\left[0, 1 \right]$ and each basis function $S_{i}(x)$ has the same shape for different partition numbers, as shown in Figure~\ref{fig:basis_densities_comparison} and Figure~\ref{fig:cubic_spline_activation_function}.

%\begin{figure}[!ht]
%\centering
%%\fbox{\rule[-.5cm]{0cm}{4cm} \rule[-.5cm]{4cm}{0cm} }
%\includegraphics[width=0.7\linewidth]{chapters/chapter1/figures/cardinal_splines_8_basis_functions.png}
%\caption{Eight Cardinal cubic B-spline basis functions with $z=5$.}
%\label{fig:cardinal_b_spline_8_many_shades_of_blue}
%\end{figure}

%\clearpage
\begin{figure} [!htb]
         \centering
         \includegraphics[width=0.65\textwidth]{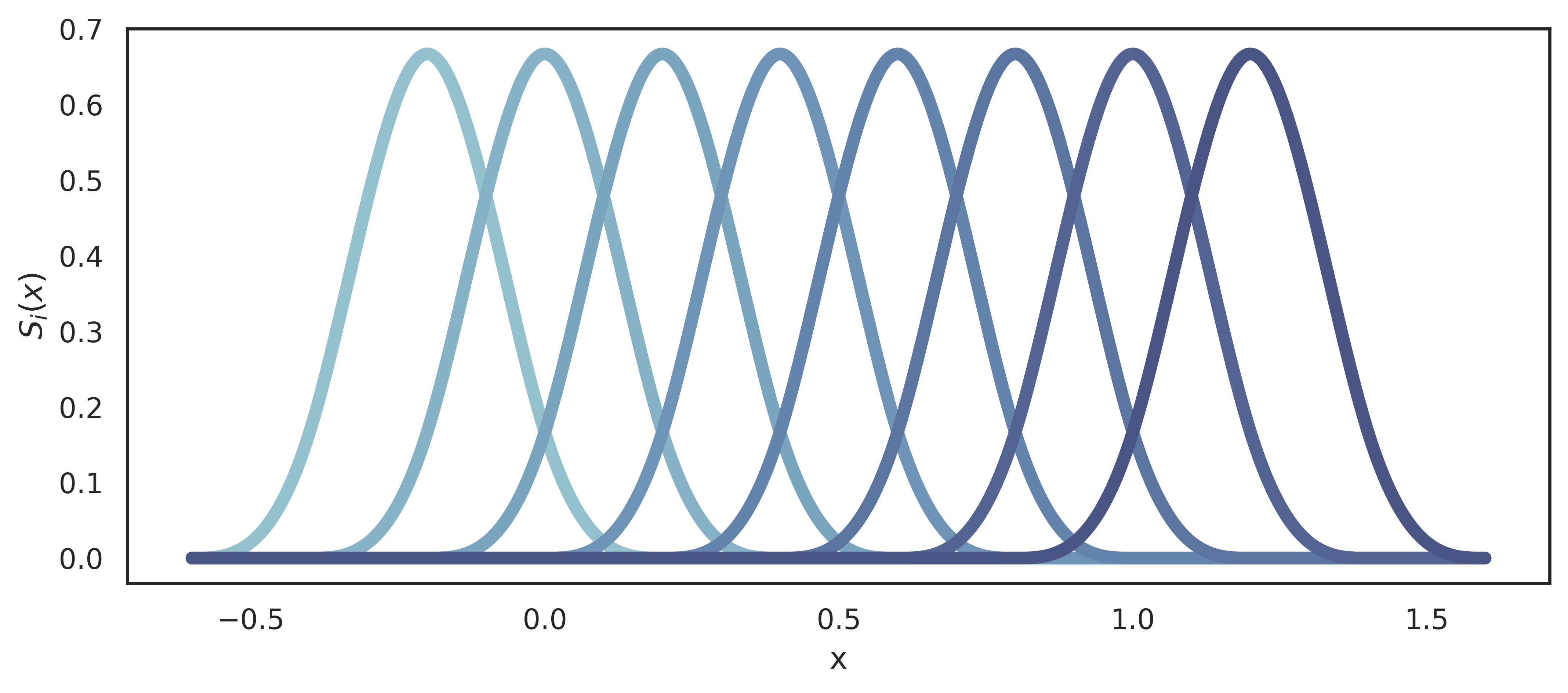}
        \caption{Uniformly spaced cardinal cubic B-splines basis functions have the same shape.}
        \label{fig:basis_densities_comparison}
\end{figure}

%\clearpage
\noindent
The basis functions extend outside the target interval $\left[0, 1 \right]$ due to an artefact of the implementation and smoothness requirements. The argument $x$ of each basis function $S_{i}(x)$ is simply scaled by $4z$ and shifted by constants $b_{i}$ before applying the same activation function $S(x)$, which is given by:
\begin{equation}
\label{eq:cubic_spline_activation}
 \footnotesize S(x) =\begin{cases} 
      \frac{1}{6} x^{3} &  0 \leq x < 1\\
      \frac{1}{6} \left[-3(x-1)^{3} +3(x-1)^{2} +3(x-1) + 1 \right] &  1 \leq x < 2\\
      \frac{1}{6} \left[3(x-2)^{3} -6(x-2)^{2} + 4 \right]  & 2 \leq x < 3\\
      \frac{1}{6} ( 4-x ) ^{3} &  3 \leq x < 4\\
      0 & otherwise. 
   \end{cases}   
\end{equation}

\clearpage
\noindent
The activation function $S(x)$ is non-zero over the domain $x \in \left[0,4\right]$, and shown in Figure~\ref{fig:cubic_spline_activation_function}. The shape of $S(x)$ resembles a Gaussian, even though the underlying function is piece-wise cubic. There are some similarities to using Gaussian kernels for function approximation, with one key difference: Cubic B-spline basis functions are zero almost everywhere, and Gaussian functions are non-zero everywhere.

\begin{figure}[!ht]
\centering
%\fbox{\rule[-.5cm]{0cm}{4cm} \rule[-.5cm]{4cm}{0cm} }
\includegraphics[width=0.5\linewidth]{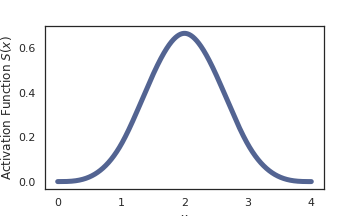}
\caption{Activation function $S(x)$ used to compute cardinal cubic B-splines.}
\label{fig:cubic_spline_activation_function}
\end{figure}

\noindent
A two-layer neural network with activation function $S(x)$ can implement cardinal cubic B-splines. Keep in mind that only the coefficients $\theta_{i}$ are trainable, thus optimising $f(x) = \sum_{i=1}^{N} \theta_{i} S\left(w_{i}x + b_{i}\right)$ is linear. In contrast, splines that unevenly partition the input space do not permit such a straightforward implementation.

Using cardinal B-splines instead of arbitrary and trainable sub-interval partitions and knots makes optimisation and implementation easier~\citep{douzette2017b}. Optimising partitions is non-linear, but optimising only coefficients $\theta_{i}$ (also called control points) is linear and thus convex. Suppose one changes one parameter $\theta_{i}$ in a cardinal cubic B-spline $f(x)$; then the model only changes on a small region as shown in Figure~\ref{fig:cardinal_b_splines_localised_parameter}. Unlike trigonometric and polynomial functions, cardinal cubic B-splines have parameters that affect the spline \textit{locally}.

\begin{figure}[!htb]
\centering
%\fbox{\rule[-.5cm]{0cm}{4cm} \rule[-.5cm]{4cm}{0cm} }
\includegraphics[width=0.8\linewidth]{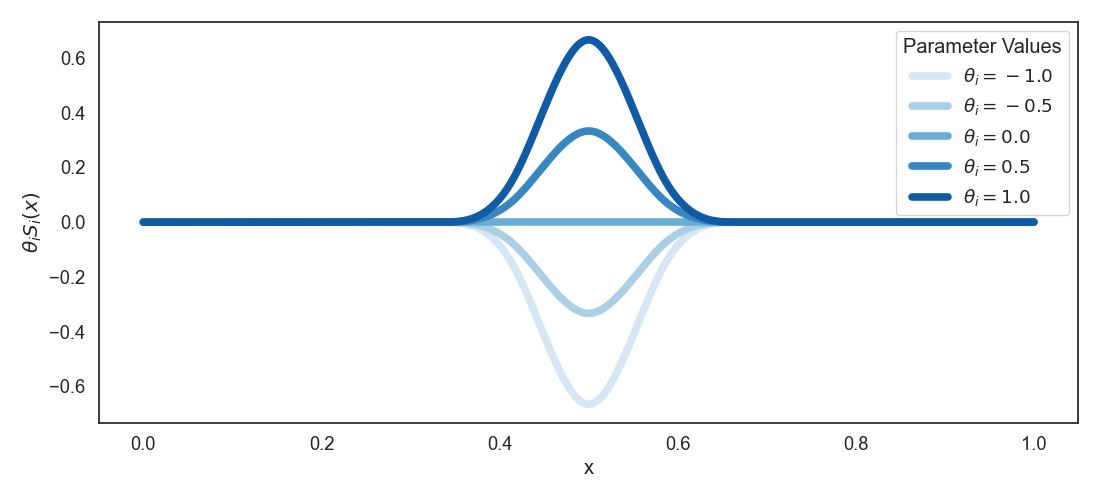}
\caption{B-spline basis function parameters are localised.}
\label{fig:cardinal_b_splines_localised_parameter}
\end{figure}

%\section{Properties of Cardinal B-splines}
%\label{sec:properties_of_cardinal_splines}

%\subsection{Sparsity}
\begin{proposition}[Sparsity]
\label{prop:z_spline_sparsity}
%%%%%%%%%%%%%%%%%%%%%%%%%%%%%%%%%%%%%%%%%%%%%%%%%%%%%%%%%%%%%%%%%%%%%%%%%%%%%%%%%%%%%%%%%
%For any $x \in D(f) \subset R$ and bounded trainable parameters $\theta_{i}$ with index set $\Theta$, the gradient vector of trainable parameters for a $z$-density B-spline function, $f$, with non-zero entries denoted $\lVert \nabla_{\theta} f(x) \rVert_{0}$ is bounded:

Let $f(x)$ be a $z$-density B-spline function, defined on the domain $[0,1] \subset R$, with trainable parameters $\theta_{i}$ and partition number $z$. Let $\lVert \nabla_{\theta} f(x) \rVert_{0}$ denote the number of non-zero components of the gradient vector w.r.t. trainable parameters. For any $x \in D(f)$, the number of non-zero components is bounded:

$$ 
\lVert \nabla_{\theta} f(x) \rVert_{0} 
%:=  \sum_{i=1}^{z}  \rho \left(\frac{\partial }{\partial \theta_{i}}f(x),0  \right) 
\leq 4
%\; \forall x \in D(f)
$$
\end{proposition}

%\noindent
%The gradient vector is zero for nearly all trainable parameters since few basis functions are non-zero, as shown in Figure~\ref{fig:fig_proof_properties_1_2}. Sparse gradients are mostly zero, leaving most parameters unchanged during training. Sparsity also permits more efficient function evaluation. This study successfully utilised sparsity and only evaluated non-zero basis functions in the proposed architectures. One can compute only the non-zero indices with scaling, shifting, rounding, and embedding layers. With an efficient implementation, the time complexity is kept constant w.r.t. $z$ for large models with a large density of basis functions (see Section~\ref{sec:implementation_of_abel_spline}, in Section~\ref{chap:abel_spline}). 

%\newpage
\begin{proof} Let $f(x)$ be a $z$-density B-spline function from Definition~\ref{def:lambda_desnity_b_spline}. Consider the components of the vector $\nabla_{\theta} f(x)$ obtained from the gradient operator, which is simply 

$$\frac{\partial}{\partial \theta_{i}} f(x) 
= \frac{\partial}{\partial \theta_{i}} 
\sum_{j=1}^{4z+3} \theta_{j} S_{j}(x) = S_{i}(x)$$

%All basis functions $S_{i}(x)$ are bounded because the activation function $S(x)$ is bounded by some constant. Furthermore, 
%\newpage
\noindent
Inspecting each basis function $S_{i}(x)$ shows that, at most, four basis functions are non-zero for a fixed $x \in [0,1]$, as visualised in Figure~\ref{fig:fig_proof_properties_1_2}. It follows that:

$$
\lVert \nabla_{\theta} f(x) \rVert_{0}
\leq 4
$$

Thus, the number of non-zero entries is bounded.
\end{proof}

\noindent
\begin{remark}
    A cubic function has at most four non-zero polynomial coefficients. Correspondingly, there are at most four non-zero B-spline basis functions. Sparse gradient vectors leave most weights unchanged and can prevent catastrophic interference. Predictable sparsity also permits efficient model implementations. 
\end{remark} 

\begin{figure}[!htb]
\centering
%\fbox{\rule[-.5cm]{0cm}{4cm} \rule[-.5cm]{4cm}{0cm} }
\includegraphics[width=0.7\linewidth]{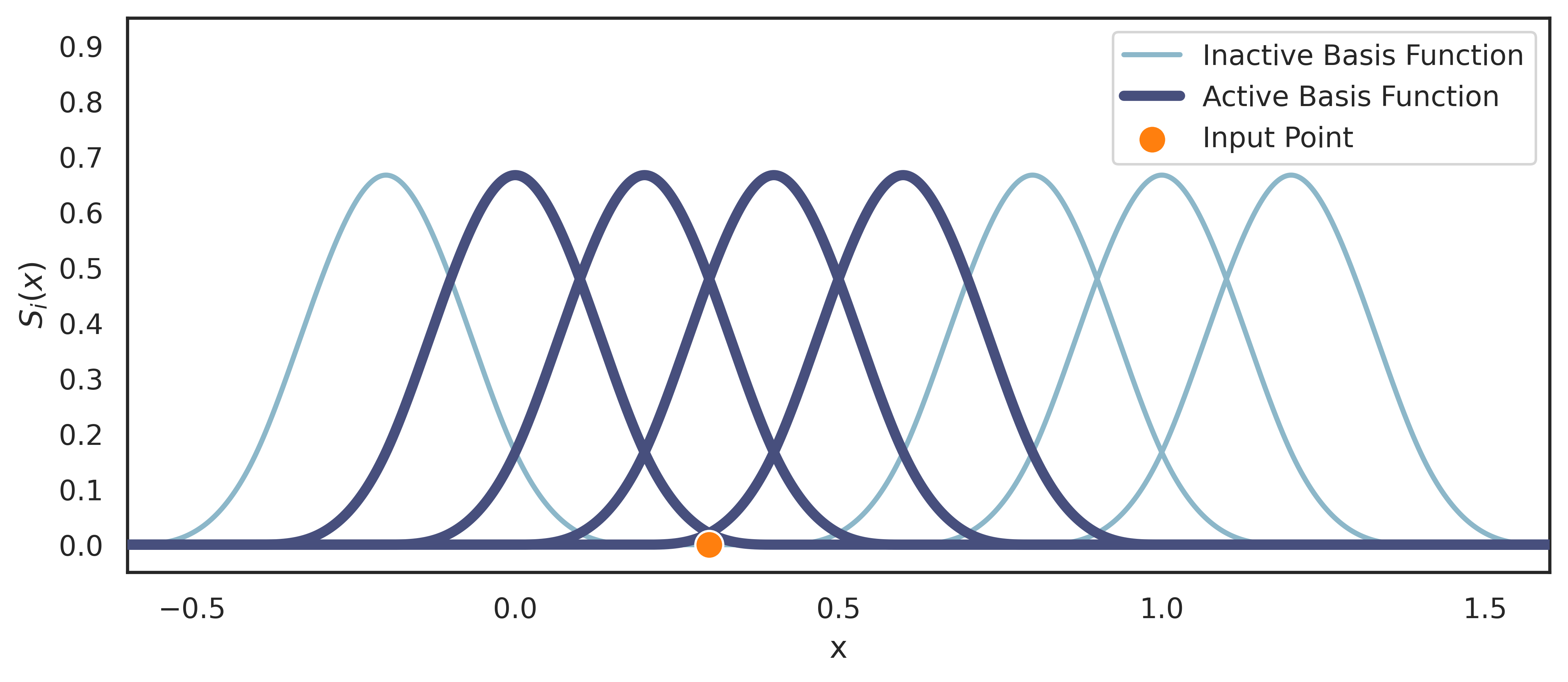}
\caption{Visual proof of sparse and bounded gradients for cardinal B-splines.}
\label{fig:fig_proof_properties_1_2}
\end{figure}

%\subsection{Bounded Gradients}
\begin{proposition}[Bounded Parameter Gradient]
\label{prop:z_spline_bounded_gradient}

%%%%%%%%%%%%%%%%%%%%%%%%%%%%%%%%%%%%%%%%%%%%%%%%%%%%%%%%%%%%%%%%%%%%%%%%%%%%%%%%%%%%%%%%%
%For any $x \in D(f) \subset R$ and bounded trainable parameters $\theta_{i}$ with index set $\Theta$, the gradient vector of trainable parameters for a $z$-density B-spline function, $f$, is bounded:

Let $f(x)$ be a $z$-density B-spline function, defined on the domain $[0,1] \subset R$, with trainable parameters $\theta_{i}$ and partition number $z$. For any $x \in D(f)$, the gradient w.r.t. trainable parameters, $\nabla_{\theta} f(x)$, is bounded:
\begin{equation*}
\lVert \nabla_{\theta} f(x)\rVert_{1} 
= \sum_{i=1}^{4z+3}  \left| \frac{\partial f}{\partial \theta_{i}} (x) \right| 
< 4
%\; \forall x \in D(f)
\end{equation*}

\end{proposition}

\begin{proof} Consider the components of the vector $\nabla_{\theta} f(x)$ obtained from the gradient operator, which is simply:
\begin{equation*}
 \frac{\partial}{\partial \theta_{i}} f(x) 
= \frac{\partial}{\partial \theta_{i}} 
\sum_{i=1}^{4z+3} \theta_{i} S_{i}(x) = S_{i}(x)   
\end{equation*}

\noindent
All basis functions $S_{i}(x)$ are bounded, since for any $x \in \mathbb{R}$, the activation function $S(x)$ is always less than $1$, as visualised in Figure~\ref{fig:cubic_spline_activation_function} and Figure~\ref{fig:fig_proof_properties_1_2}. It follows that:
\begin{equation*}
 \lVert \nabla_{\theta} f(x)\rVert_{1} 
= \sum_{i=1}^{4z+3}  \left| \frac{\partial f}{\partial \theta_{i}} (x) \right| 
= \sum_{i=1}^{4z+3}  \left| S_{i} (x) \right| 
< 4   
\end{equation*}

Thus, the parameter gradient is bounded.
\end{proof}
\noindent 
%%%%\begin{remark}
%%%%    Bounded and sparse gradients suggest that the model is numerically stable, and small perturbations will not excessively affect the model's learned values. The boundedness of the parameter gradient is independent of the partition number or number of basis functions. 
%%%%\end{remark}

%In contrast, Fourier or polynomial functions are sensitive to parameter changes. Fourier and polynomial functions have potentially unbounded gradients since nearly all the basis functions are non-zero. Increasing the number of basis functions without rescaling with absolutely convergent series to dampen higher frequencies can lead to potentially large gradients.

%\clearpage
%\subsection{Uniform Trainability}

\begin{proposition}[Trainability of Cardinal B-splines]
\label{prop:z_spline_uniform_trainability}
Let $f(x)$ be a $z$-density B-spline function, defined on the domain $\left[0,1\right] \subset R$, with trainable parameters $\theta_{i}$. The gradient w.r.t. trainable parameters, $\nabla_{\theta} f(x)$, is non-zero:
\begin{equation*}
    \nabla_{\theta} f(\theta, x) \neq \Vec{0}, \; \forall x \in D(f)
    %, \; \forall \theta \in \Theta
\end{equation*}
\end{proposition}

%\newpage
\begin{proof}
Consider the components of the vector $\nabla_{\theta} f(x)$ obtained from the gradient operator, which is simply 
\begin{equation*}
\frac{\partial}{\partial \theta_{i}} f(x) 
= \frac{\partial}{\partial \theta_{i}} 
\sum_{j=1}^{4z+3} \theta_{j} S_{j}(x) = S_{i}(x)    
\end{equation*}

%All basis functions $S_{i}(x)$ are bounded because the activation function $S(x)$ is bounded by some constant. Furthermore, 
\noindent
If $x \in \left[ 0, 1\right]$ is inside the support interval or at the boundary, then at least three basis functions are non-zero:
\begin{equation*}
    \nabla_{\theta} f(\theta, x) \neq \Vec{0}, \; \forall x \in D(f), \; 
    \forall \theta \in \Theta
\end{equation*}

Thus, $z$ density B-splines are uniformly trainable on the unit interval.
\end{proof}

%\subsection{Absolute Distal Orthogonality}

%For any $\delta>0$ there exists a $z\geq \delta^{-1}$ ABEL-Spline function $A(x)$, defined on the domain $x \in \left[ 0,1 \right]^{n}$, with trainable parameters $\theta \in \Theta$ and  For any $x,y  \in D(f)$:
\begin{proposition}[absolute distal orthogonality]
\label{prop:z_spline_distal_orthogonality}  
Let $f(x)$ be a $z$-density B-spline defined on $[0,1]$, with partition number $z$ and trainable parameters $\theta$, then $\forall \; x,y \in D(f)$:
\begin{equation*}
 |x - y| > z^{-1} 
\implies   \nabla_{\theta} f(x) \cdot \nabla_{\theta} f(y) = 0  
\end{equation*}
\end{proposition}

\begin{proof}
Let $f(x)$ be a $z$-density B-spline function, defined on the domain $[0,1] \subset R$, with trainable parameters $\theta_{i}$ and partition number $z$. Let $x,y \in D(f)$, and assume that:
$$
 |x-y| > z^{-1}
$$

The $z$-density B-spline $f(x)$ partitions the unit interval into $4z$ sub-intervals of equal length ($\frac{1}{4z}$). A visual aid is provided in Figure~\ref{fig:fig_proof_properties3}. Two distant points are shown in orange. The active or non-zero basis functions associated with each point are shown in dark blue. The inactive or zero basis functions are shown in light blue.

If the distance between $x$ and $y$ is larger than \textit{four} sub-intervals, then there are no basis functions in common. Therefore, if $ |x - y| > 4(\frac{1}{4z}) = z^{-1}$, then the gradients must be orthogonal with no overlapping parameters. Thus, $\nabla_{\theta} f(x) \cdot \nabla_{\theta} f(y) = 0$
\end{proof}

\begin{figure}[!htb]
\centering
%\fbox{\rule[-.5cm]{0cm}{4cm} \rule[-.5cm]{4cm}{0cm} }
\includegraphics[width=0.7\linewidth]{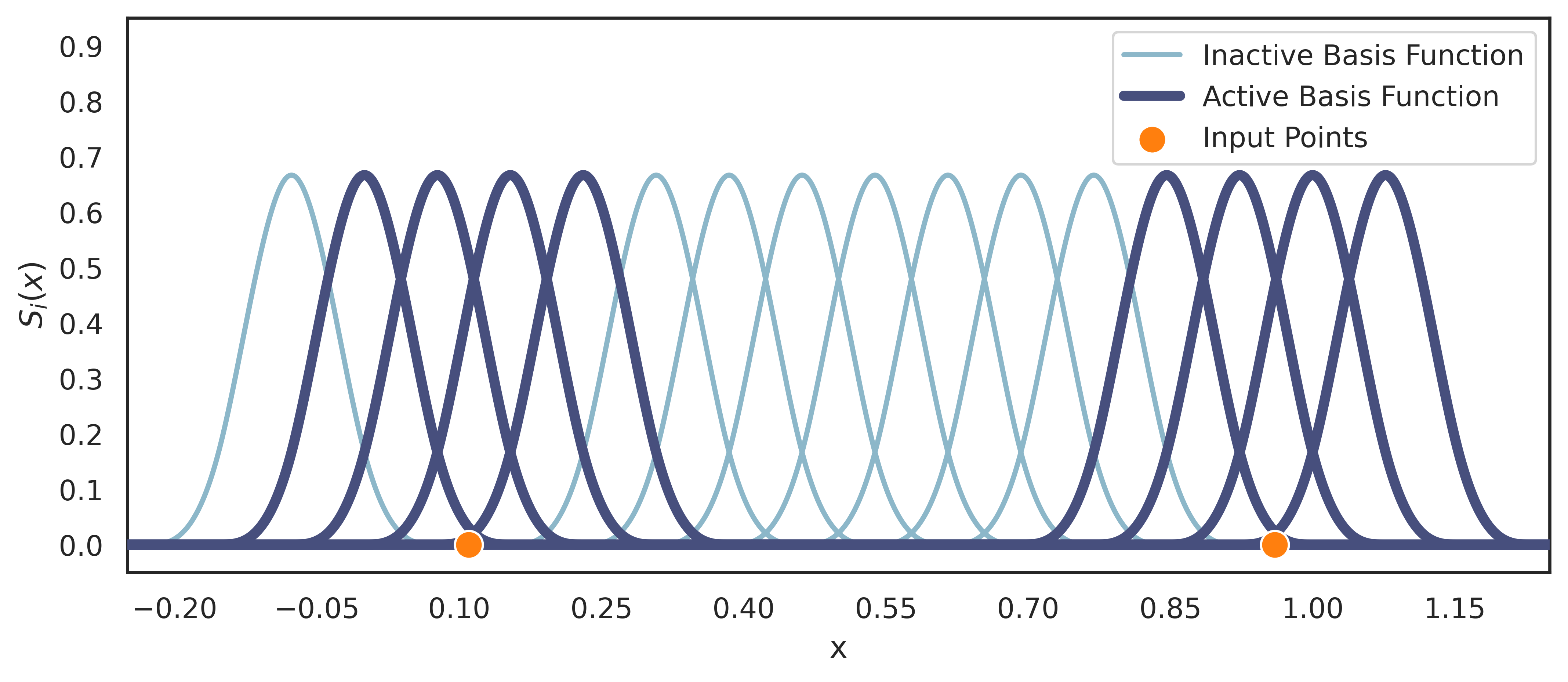} 
\caption{Visual proof of distal orthogonality for cardinal B-splines.}
\label{fig:fig_proof_properties3}
\end{figure}

\subsection{Spline ANN}

Incorporating splines into ANNs has been studied to some extent. \citet{lane1991multi} published a paper at the third NIPS (currently NeurIPS) conference that employed B-splines to model receptive fields in MLP models. Some investigations used uniformly partitioned splines to implement adaptive and trainable activation functions named spline activation functions (SAFs) in ANNs~\citep{scardapane2017learning_workshop}. Studies have also been conducted on trainable or uneven partition splines that allow the partitions (or knots) to be trained with gradient descent~\citep{douzette2017b}. This study uses single-variable splines to develop multi-variable spline architectures that are \textbf{min-distal orthogonal}. 

The $z$-Spline ANNs (or simply Spline ANNs) in this study are sums of single-variable $z$-density B-spline functions and can adequately approximate sum-decomposable target functions. General additive models approximate multi-variable functions as sums of single-variable functions and are often used in statistics~\citep{Wood_2004, hastie2017generalized, molnar2020interpretable}. Note that the same partition number is used for each $z$-density B-spline in a $z$-Spline ANN.

\begin{definition}[$z$-Spline ANN]
\label{def:lambda_sam}
A $z$-Spline ANN model $f(x)$, defined on $\left[0,1\right]^{n}$, with trainable parameters $\theta$ and partition number $z \in \mathbb{N}$, is a sum of $n \in \mathbb{N}$ single-variable $z$-density B-splines in each variable:
\begin{equation*} %\label{eq1}
\begin{split}
f(x)  
= \sum^{n}_{j=1} f_{j}( x_{j} ) 
= \sum^{n}_{j=1} \sum_{i=1}^{4z+3}   \theta_{i,j} S_{i,j}(  x_{j} )
= \sum^{n}_{j=1} \sum_{i=1}^{4z+3}   \theta_{i,j} S((4z) x_{j} + (4 - i))
\end{split}
\end{equation*}
\end{definition}

\noindent
where $\theta_{i,j} $ corresponds to trainable parameters or coefficients, and $S_{i,j}(  x_{j} )$ denotes the fixed basis functions for cardinal cubic B-splines. $S(x)$ is the cubic spline activation function. The $z$-Spline ANNs are uniformly trainable and min-distal orthogonal ($z\geq\delta^{-1}$) models with sparse and bounded parameter gradients. See Section~\ref{sec:properties_of_cardinal_splines} for explanations and proofs of the properties of cardinal cubic B-splines, and Appendix~\ref{sec:spline_additive_models} for the corresponding proofs for $z$-Spline ANNs. 

\begin{remark}
    $z$-Spline ANNs have linear time and space complexity in the input dimension. Optimising $z$-Spline ANNs parameters with a convex loss function yields a convex optimisation problem with a global optimum that can be reached with gradient descent optimisation methods. 
\end{remark}

\clearpage
$z$-Spline ANNs can be helpful in specific applications. If the target function is a sum of single-variable functions, then $z$-Spline ANNs would be ideal, but many problems have more complicated target functions. Unfortunately, $z$-Spline ANNs are not universal function approximators, but $z$-Spline ANNs are extended to provable universal function approximators defined below in Section~\ref{sec:abel_spline_main_body}, with accompanying proofs in Appendix~\ref{sec:universal_function_approximation} and Appendix~\ref{sec:properties_of_abel_spline}.

\subsection{ABEL-Splines}
\label{sec:abel_spline_main_body}

Antisymmetric bounded exponential layer spline ANNs (ABEL-Splines) can approximate any continuous function.\footnote{ABEL-Splines are very \textit{able}.} ABEL-Splines are designed to inherit properties from $z$-Spline ANNs that are atypical of most universal function approximators:
\begin{enumerate}[itemsep=0.1mm,parsep=0.1cm]
    \item \textbf{Sparse activity}---most neural units are zero and inactive.
    \item \textbf{Bounded parameter gradients} regardless of model size.
    \item \textbf{Uniformly trainable} anywhere in the domain.
    \item \textbf{Min-distal orthogonality} with $z\geq\delta^{-1}$.
\end{enumerate}

\begin{definition}[ABEL-Spline]
\label{def:lambda_abel_spline}
%Any $z$-ABEL-Spline $A(x)$ is a function defined on $x \in \left[0,1\right]^{n}$, 
Let $A(x)$ be an ABEL-Spline function, defined on the domain $x \in \left[ 0,1 \right]^{n}$, with trainable parameters $\theta \in \Theta$ and partition number $z \in \mathbb{N}$. Then, there exists $\mathcal{K} \in \mathbb{N}$, and multi-variable $z$-Spline ANN functions $F(x),G_{k}(x), H_{k}(x)$ such that:
\begin{equation*} %\label{eq1}
\begin{split}
A(x) 
:= F(x) + 
\sum_{k = 1}^{\mathcal{K}} \frac{1}{k^{2}} \bigg( \exp(G_{k}(x)) 
-  \exp(H_{k}(x)) \bigg)
\end{split}
\end{equation*}
\end{definition}

\noindent
ABEL-Splines can equivalently be given in terms of single-variable $z$-density B-spline functions $f_{j}(x_{j}),g_{k,j}(x_{j})$, and $h_{k,j}(x_{j})$ such that:
\begin{equation*} %\label{eq1}
\begin{split}
A(x) 
    = \sum_{j=1}^{n} f_{j}(x_{j}) + 
\sum_{k = 1}^{\mathcal{K}} \frac{1}{k^{2}} \bigg( \exp( \Sigma_{j=1}^{n} g_{k,j}(x_{j})) 
-  \exp( \Sigma_{j=1}^{n} h_{k,j}(x_{j})) \bigg)
\end{split}
\end{equation*}

\begin{remark}
    Multi-dimensional outputs are treated as separate scalar functions, approximated with the outlined schema, and are independent of each other. 
\end{remark}

The absolutely convergent series of scale factors $k^{-2}$ was chosen for numerical stability and to ensure the model is absolutely convergent as $\mathcal{K} \rightarrow \infty$. A series $\sum_{i}a_{i}$ is absolutely convergent if $\sum_{i}|a_{i}|$ converges to a well-defined limit~\citep{rudin1976principles}. Another feature is that the series of scale factors also breaks the symmetry that would otherwise exist between terms. The `residual' or direct $F(x)$ $z$-Spline ANN is included to ensure that ABEL-Splines can easily represent and learn any sum-decomposable function.

The relevant proofs of universal function approximation can be found in Appendix~\ref{sec:universal_function_approximation} and Appendix~\ref{sec:properties_of_abel_spline}, which gives the proofs for the other listed properties of ABEL-Splines.

\section{Experimentation}

Experimentation is limited to a tractable subset to demonstrate mathematically proven properties. Model perturbation and distal interference is estimated using randomly generated training data (Section~\ref{sec:model_perturbation_distal_interference}). In Section~\ref{sec:two_dimensional_demonstration}, a 2D regression problem and a synthetic sequential learning problem (without and with pseudo-rehearsal) demonstrate the limits of min-distal orthogonal models. The relevant code is available in a public \href{https://github.com/hpdeventer/catastrophic-interference-2023-2024}{GitHub repository}.\footnote{\url{https://github.com/hpdeventer/catastrophic-interference-2023-2024}}

\subsection{Considered Models}
\label{subsec_considered_models_2d_demo}

The five models considered for experimentation have two-dimensional input. Low-dimensional models are necessary since a direct comparison is made with a lookup table model.

The two ReLU ANN models are included for the sake of comparison. The deep- and wide ReLU ANNs have a shift and scaling input layer to map values from $[0,1] \mapsto [-1,+1]$, which is close to conventional approaches. Trainable parameters are initialised with `GlorotUniform.' The wide ReLU ANN model has $1000$ hidden ReLU units in one hidden layer. The deep ReLU ANN model has $8$ hidden layers with $16$ ReLU units in each hidden layer. Dense connections are used throughout.
 
The three partition-based models with a chosen partition number $z=20$ include a max-distal orthogonal lookup table model, a Spline ANN ($z$-Spline ANN), and an ABEL-Spline. The choice of partition number $z=20$ is for visual inspection of the lookup table model. Smaller partition numbers make it difficult to discern visually if the lookup table can learn the target function. The partition-based models have trainable parameters initialised with `RandomUniform.' The ABEL-Spline model has $\mathcal{K}=6$ positive and negative exponentials each. There are five models used throughout experimentation:
\begin{enumerate}
    \item Wide ReLU ANN
    \item Deep ReLU ANN
    \item ABEL-Spline ($z = 20$)
    \item $z$-Spline ANN ($z = 20$)
    \item Lookup Table ($z = 20$)
\end{enumerate}

\subsection{Model Perturbation and Distal Interference}
\label{sec:model_perturbation_distal_interference}

Mathematically proven properties are tested numerically using Monte Carlo integration. 100 independent trials are performed. In each trial, one training data point $v \in D$ is sampled uniformly from the domain $D=[ 0,1]^{n}$, and a target value is sampled from a normal distribution. Each randomly initialised model $f_{t}$ is trained using the Adam optimiser (default learning rate = 0.001) for one epoch, yielding an updated model $f_{t+1}$. 

The model perturbation $\lVert f_{t+1} - f_{t} \rVert_{1} = \int_{D} |f_{t+1}(x) - f_{t}(x)| \,dx$ is estimated numerically using Monte Carlo integration and 100,000 uniformly sampled points from the domain $D=[ 0,1]^{n}$. The mean model perturbation and standard deviation over the 100 trials are presented in Table~\ref{tab:model_perturbation}. The wide ReLU ANN exhibited the largest model perturbation. The ReLU models' perturbation is orders of magnitude larger than the partition-based models. The lookup table exhibits the smallest model perturbation, followed by the Spline ANN and ABEL-Spline models with increasing model perturbation. 
\begin{table}[!htb]
\centering
\small
\begin{tabular}{|l|c|}
\hline
\textbf{Models} & Model Perturbation \\
\hline
Wide ReLU ANN & 2.16e-2 ($\pm$ 7.78e-4) \\
Deep ReLU ANN & 3.34e-3 ($\pm$ 5.11e-4) \\
Spline ANN (z=20) & 9.78e-5 ($\pm$ 4.17e-6) \\
ABEL-Spline (z=20) & 3.89e-4 ($\pm$ 1.70e-5) \\
Lookup Table (z=20) & 2.50e-6 ($\pm$ 1.37e-7) \\
\hline
\end{tabular}
\caption{Mean and standard deviation of model perturbation.}
\label{tab:model_perturbation}
\end{table}

The distal interference $\int_{D_{v}} |f_{t+1}(x) - f_{t}(x)| \,dx$ is estimated numerically using Monte Carlo integration over the set of distal points $D_{v}=\{ x \mid d(x,v)>\delta, \; \forall x,v \in D \}$ with 100,000 uniformly sampled points from the domain $D=[ 0,1]^{n}$. The mean distal interference and standard deviation over the 100 trials are presented in Table~\ref{tab:max_distal_interference} and Table~\ref{tab:min_distal_interference}. The mean max-distal interference in Table~\ref{tab:max_distal_interference} shows that the lookup table has zero distal interference for $\max_{i} (|x_{i} - v_{i})| > z^{-1}=0.05$, but not for a smaller threshold of $0.01$.   
\begin{table}[!htb]
\centering
\small
\begin{tabular}{|l|c|c|c|}
\hline
\textbf{Models} & $\max_{i} (|x_{i} - v_{i})| > 0.1$ & $\max_{i} (|x_{i} - v_{i})| > 0.05$ & $\max_{i} (|x_{i} - v_{i})| > 0.01$\\
\hline
Wide ReLU ANN & 1.98e-2 ($\pm$ 7.00e-4) & 2.11e-2 ($\pm$ 7.29e-4) & 2.16e-2 ($\pm$ 7.76e-4) \\
Deep ReLU ANN & 3.16e-3 ($\pm$ 4.73e-4) & 3.29e-3 ($\pm$ 5.00e-4) & 3.33e-3 ($\pm$ 5.11e-4) \\
Spline ANN (z=20) & 7.91e-5 ($\pm$ 2.90e-6) & 8.83e-5 ($\pm$ 3.45e-6) & 9.70e-5 ($\pm$ 4.15e-6) \\
ABEL-Spline (z=20) & 3.14e-4 ($\pm$ 1.18e-5) & 3.51e-4 ($\pm$ 1.41e-5) & 3.85e-4 ($\pm$ 1.69e-5) \\
Lookup Table (z=20) & \textbf{0.0} ($\pm$ \textbf{0.0}) & \textbf{0.0} ($\pm$ \textbf{0.0}) & 2.16e-6 ($\pm$ 1.61e-7) \\
\hline
\end{tabular}
\caption{Mean and standard deviation of max-distal interference.}
\label{tab:max_distal_interference}
\end{table}

The mean min-distal interference in Table~\ref{tab:min_distal_interference} shows that the Spline ANN, ABEL-Spline, and lookup table models have zero min-distal interference when $\min_{i} (|x_{i} - v_{i})| > z^{-1}=0.05$. The numerical results support the mathematical analysis and proof of distal orthogonality.   
\begin{table}[!htb]
\centering
\small
\begin{tabular}{|l|c|c|c|}
\hline
\textbf{Models} & $\min_{i} (|x_{i} - v_{i})| > 0.1$ & $\min_{i} (|x_{i} - v_{i})| > 0.05$ & $\min_{i} (|x_{i} - v_{i})| > 0.01$\\
\hline
Wide ReLU ANN & 1.08e-2 ($\pm$ 1.37e-3) & 1.56e-2 ($\pm$ 8.82e-4) & 2.03e-2 ($\pm$ 6.84e-4) \\
Deep ReLU ANN & 2.01e-3 ($\pm$ 3.16e-4) & 2.61e-3 ($\pm$ 3.93e-4) & 3.18e-3 ($\pm$ 4.84e-4) \\
Spline ANN (z=20) & \textbf{0.0} ($\pm$ \textbf{0.0}) & \textbf{0.0} ($\pm$ \textbf{0.0}) & 5.70e-5 ($\pm$ 3.62e-6) \\
ABEL-Spline (z=20) & \textbf{0.0} ($\pm$ \textbf{0.0}) & \textbf{0.0} ($\pm$ \textbf{0.0}) & 2.26e-4 ($\pm$ 1.47e-5) \\
Lookup Table (z=20) & \textbf{0.0} ($\pm$ \textbf{0.0}) & \textbf{0.0} ($\pm$ \textbf{0.0}) & 1.00e-6 ($\pm$ 1.56e-7) \\
\hline
\end{tabular}
\caption{Mean and standard deviation of min-distal interference.}
\label{tab:min_distal_interference}
\end{table}

\subsection{Two-Dimensional Demonstration}
\label{sec:two_dimensional_demonstration}

This section discusses the experiments on a two-dimensional regression task designed to illustrate universal function approximation and continual learning capabilities of the ABEL-Spline architecture. Section~\ref{subsec_simple_2d_regression} outlines the 2D regression problem. A synthetic 2D continual learning problem is discussed in Section~\ref{subsec_sequential_learning_with_catastrophic_forgetting}. Section~\ref{subsec_sequential_learning_with_pseudo_rehearsal} demonstrates the effect of pseudo-rehearsal in a 2D setting.

\subsubsection{Regression Task}
\label{subsec_simple_2d_regression}
In this experiment, the target function is chosen to demonstrate and visualise the limited expressive power of $z$-Spline ANNs in comparison to provable universal function approximators. The target function is a product of two oscillating single-variable functions defined on $[0,1]^{2} \subset \mathbb{R}^{2}$:
\begin{equation}
    \label{eqn:2d_regression_experiment}
    y(x) = \sin(4 \pi x_{1}) \cdot \sin (4 \pi x_{2})
\end{equation}
The target function is shown in Figure~\ref{fig:results:2d_demo_target_function}, where $x$-axis and $y$-axis correspond to the two input dimensions, and the function value is depicted as a colour scale. The target function resembles a smoothed checkered pattern similar to the XOR problem. 
\begin{figure}[!ht]
    \centering
    \includegraphics[width=0.4\linewidth]{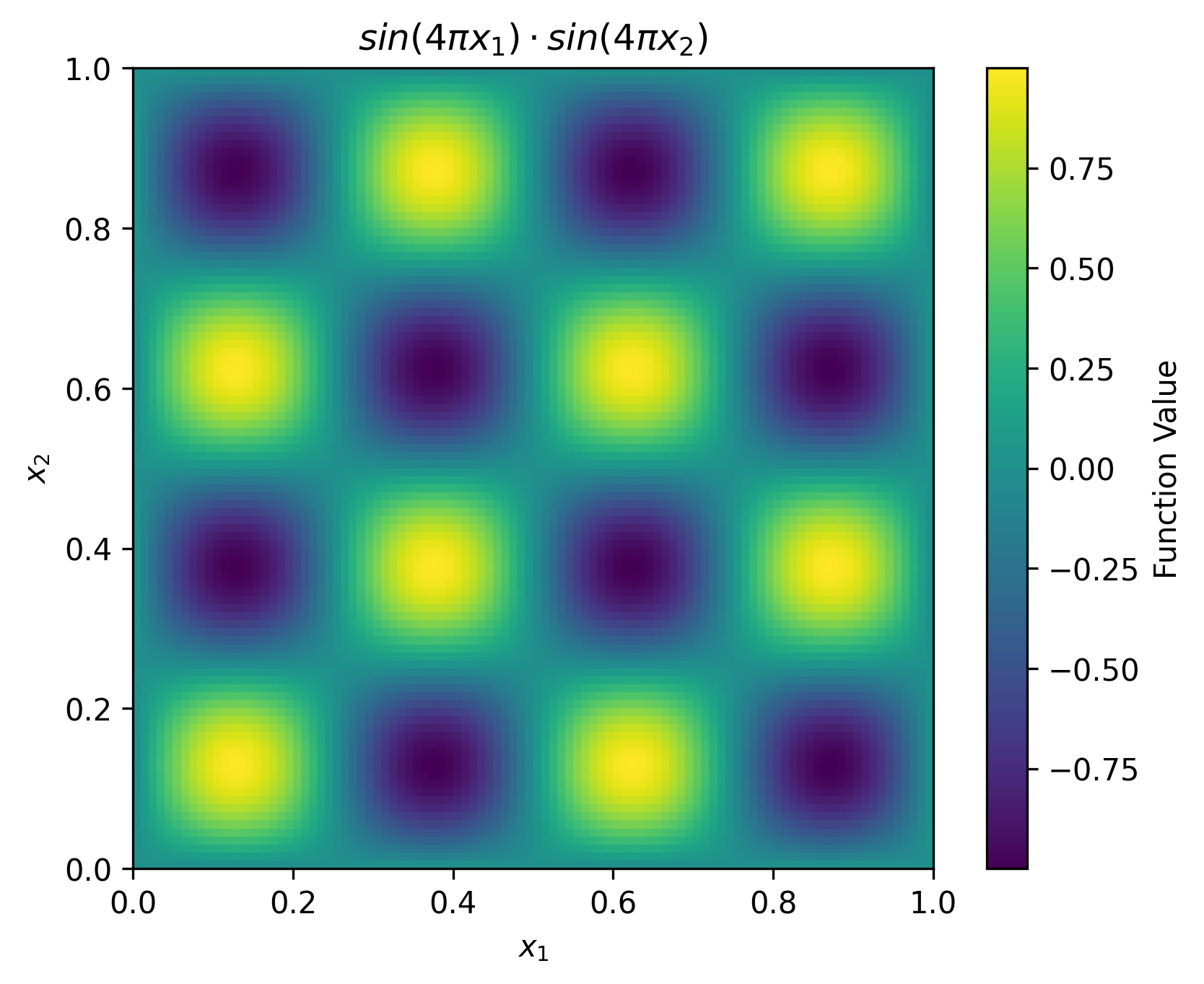}
    \caption[Target function for 2D regression experiment.]{The 2D target function.}
    \label{fig:results:2d_demo_target_function}
\end{figure}

$16000$ training points are sampled from a uniform distribution over $[0,1]^{2} \subset \mathbb{R}^{2}$, and the target values are calculated using the target function defined in Equation~(\ref{eqn:2d_regression_experiment}). Each model is trained with the Adam optimiser for $200$ epochs with a batch size of $100$ using MAE as a loss function. The model predictions after training are visualised in Figure~\ref{fig:results:2d_demo_regression_predictions}. All the models except the $z$-Spline ANN model could learn the target function. The piece-wise defined lookup table model has discontinuities between `flat' regions associated with their own parameters, leading to a pixel-like effect. The deep and wide ReLU ANNs and the ABEL-Spline models learn the target function without aberration. 
This result demonstrates the limited expressive power of $z$-Spline ANNs and the benefits of developing ABEL-Splines. ABEL-Splines have similar performance to ReLU ANNs in this specific regression problem.

\begin{figure}[!htb]
    \centering
    \includegraphics[width=0.99\linewidth]{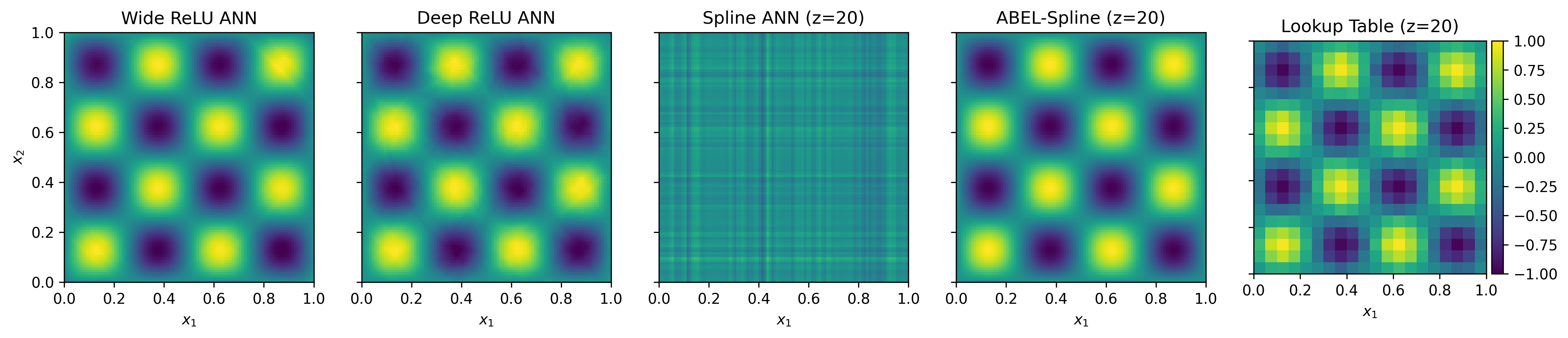}
    \caption[Model outputs for 2D regression.]{Model outputs after training on a 2D regression problem.}
    \label{fig:results:2d_demo_regression_predictions}
\end{figure}

\clearpage
\subsubsection{Sequential Learning and Catastrophic Interference}
\label{subsec_sequential_learning_with_catastrophic_forgetting}

This experiment simulates a sequential or continual learning problem that allows visualising the model outputs. The target function is the same as in Equation~(\ref{eqn:2d_regression_experiment}) and Figure~\ref{fig:results:2d_demo_target_function}. The difference is that the data are not sampled uniformly over the entire domain. Instead, the domain is partitioned into $16$ equal-sized regions, shown in Figure~\ref{fig:results:2d_demo_partitions_show}. 

The models are trained on $1000$ data points sampled inside each partition. After training for $200$ epochs on one partition using Adam and batch size of $100$, all the models are trained on the second partition and so forth. There are $16000$ data points sampled from the target function over all partitions combined. 

Figure~\ref{fig:results:2d_demo_partitions_show} shows the different partitions and the order in which the partitions are sequentially learned. Figure~\ref{fig:results:2d_demo_sequential_learning} shows the model outputs after training on all partitions in the shown (randomly sampled) order. The deleterious effect of distal interference and catastrophic interference is evident. 
\begin{figure}[!ht]
    \centering
    \includegraphics[width=0.6\linewidth]{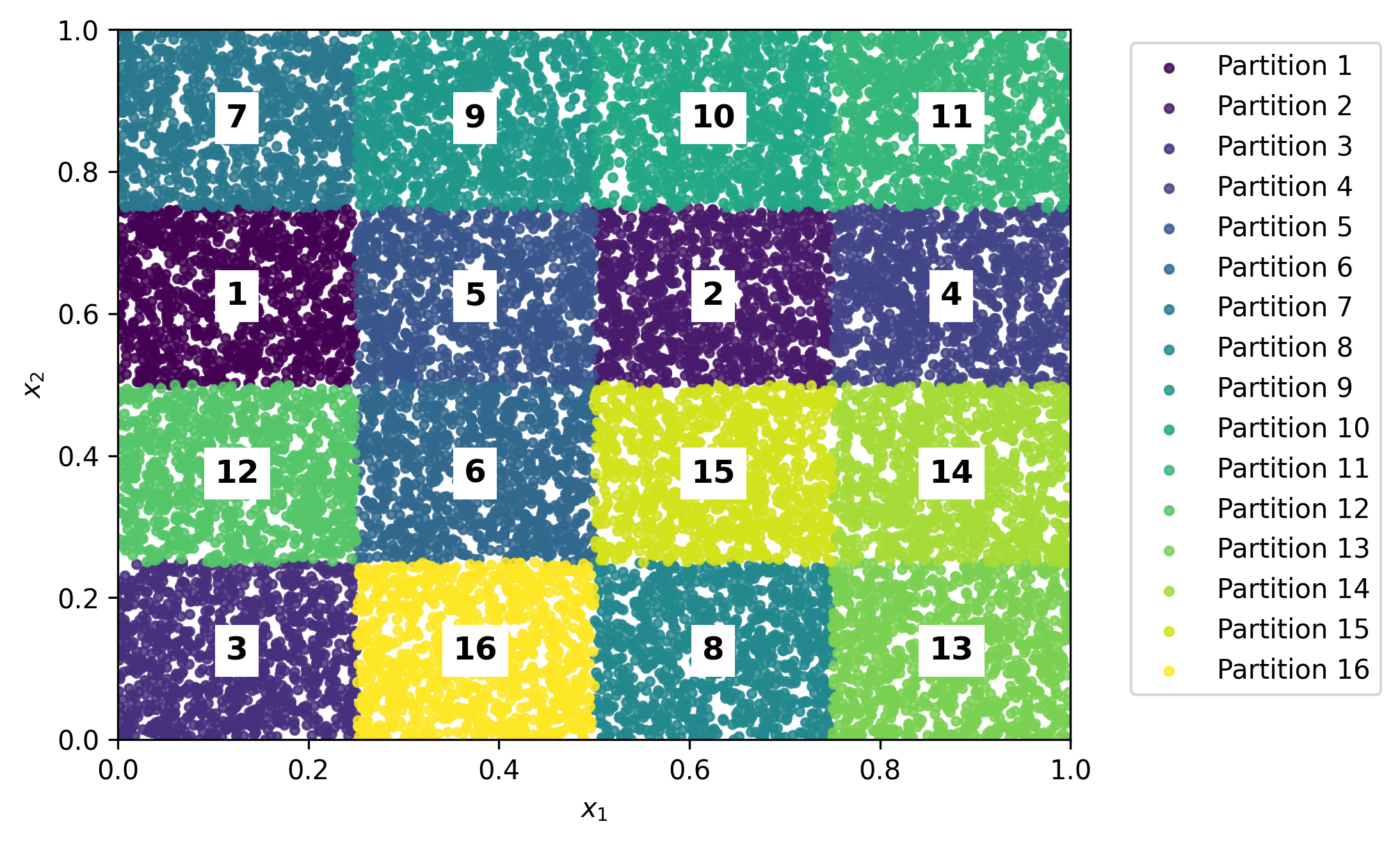}
    \caption[Partitions for 2D continual learning problem.]{Partitions and order for sequential continual learning problem.}
    \label{fig:results:2d_demo_partitions_show}
\end{figure}

Only the lookup table model has a nearly identical output, comparing Figure~\ref{fig:results:2d_demo_sequential_learning} and Figure~\ref{fig:results:2d_demo_regression_predictions} against one another, since the lookup table is a max-distal orthogonal model. All other models exhibit catastrophic interference. The ReLU and ABEL-Spline models captured the target values of the last two or so partitions. 

\begin{figure}[!ht]
    \centering
    \includegraphics[width=0.99\linewidth]{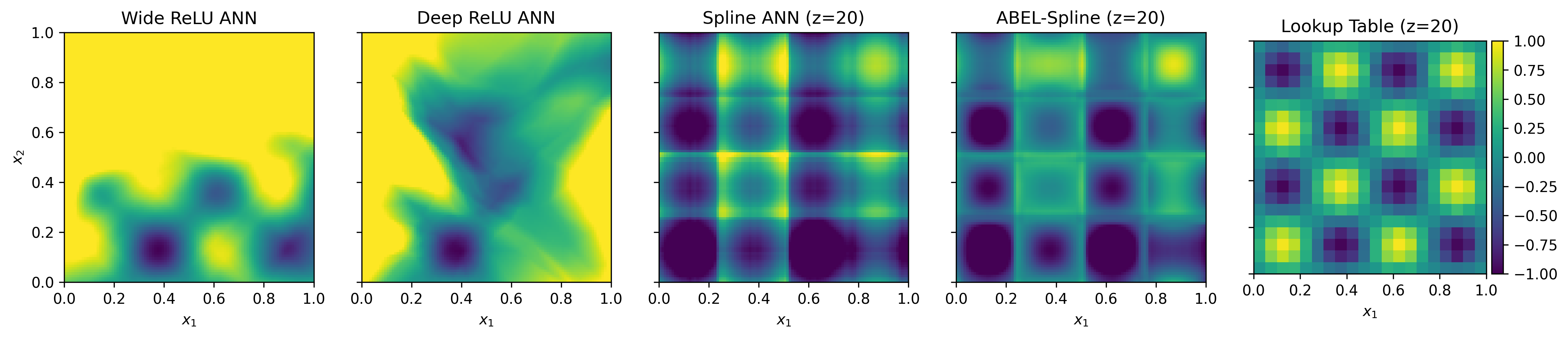}
    \caption[Model outputs for 2D continual learning problem.]{Model outputs after training on a 2D continual learning problem.}
    \label{fig:results:2d_demo_sequential_learning}
\end{figure}

\subsubsection{Sequential Learning and Pseudo-Rehearsal}
\label{subsec_sequential_learning_with_pseudo_rehearsal}

This experiment extends the experiment discussed in Section~\ref{subsec_sequential_learning_with_catastrophic_forgetting}. The critical difference is that the training data for each model is augmented with pseudo-rehearsal. Each model $f_{t}$ is randomly initialised before training on the first task. The 1000 training data points $(x,y(x))$ from each partition are combined with 1000 pseudo-rehearsal input points $u$ that are sampled from a uniform distribution over $[0,1]^{2} \subset \mathbb{R}^{2}$ and target values $f_{t}(u)$. In other words, data sampled from the target function $(x,y(x))$ for each partition is combined with data sampled from the model itself $(u,f_{t}(u))$ over the entire domain to train $f_{t}$ and obtain $f_{t+1}$ which is trained on the next partition and so forth. This training augmentation allows the capable models to learn the target function values inside each partition while retaining their values in other regions. 

The outputs of the models after training with pseudo-rehearsal are shown in Figure~\ref{fig:results:2d_demo_pseudo_rehearsal_continual_learning}. The ReLU ANNs and ABEL-Spline can learn the target function sequentially with pseudo-rehearsal. The $z$-Spline ANN is still incapable of learning the target function.

\begin{figure}[!ht]
    \centering
    \includegraphics[width=0.99\linewidth]{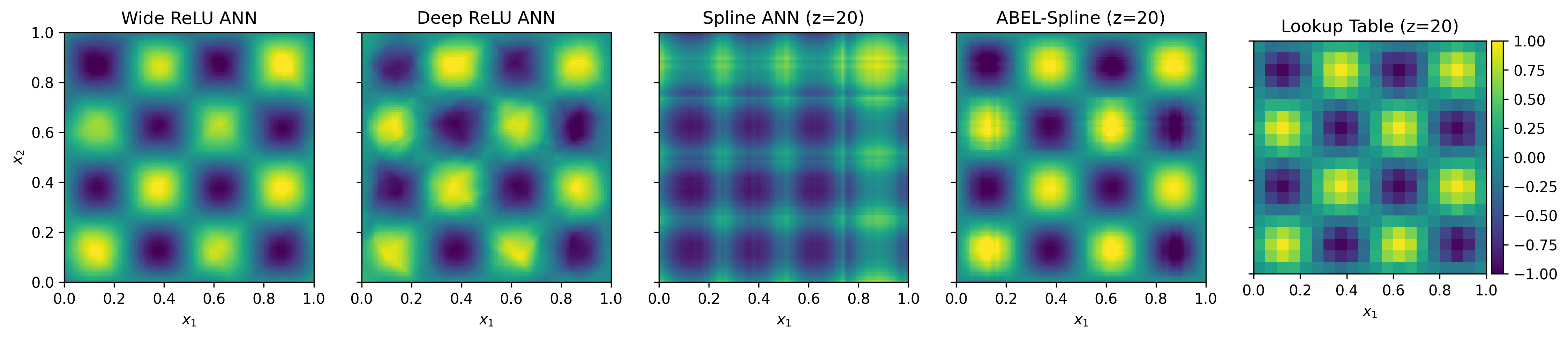}
    \caption[Model outputs for 2D continual learning problem with pseudo-rehearsal.]{Outputs after training on 2D continual learning problem with pseudo-rehearsal.}
    \label{fig:results:2d_demo_pseudo_rehearsal_continual_learning}
\end{figure}

Pseudo-rehearsal is adequate for continual learning in a low-dimensional setting. The $z$-Spline ANN did not benefit from pseudo-rehearsal because it lacks the expressive power to model the target function. Pseudo-rehearsal cannot increase the representation ability of a model. 

ABEL-Splines can learn sequentially with pseudo-rehearsal in a low-dimensional setting, but ABEL-Splines cannot learn sequentially in all circumstances. We conclude that in cases where min-distal orthogonality fails, one must use augmentation techniques.

\section{Conclusions}

It is known that overlapping representations in differentiable models such as ANNs cause catastrophic interference. This study contributed with an additional refinement: Distant (and unrelated) training inputs with overlapping representations lead to distal interference, which causes catastrophic interference over sequences of tasks. The importance of distance was not thoroughly considered in the existing literature on catastrophic interference. This study introduced the concepts of distal interference and distal orthogonality to analyse the geometry of differentiable models and their susceptibility to catastrophic interference.

\newpage
We have proven that if a model is uniformly trainable and max-distal orthogonal, then it must have an exponentially large parameter space. The contra-positive is that non-exponentially large models do not possess max-distal orthogonality or are not trainable over their entire domain. The logical implication is that polynomial complexity models are not uniformly trainable over their domain or are not max-distal orthogonal models. Training a polynomial complexity model on data in one part of the domain may perturb the model on distant points in the domain. This distal interference can cause catastrophic interference over sequentially learnt tasks. This finding undermines the potential for polynomial complexity models to learn sequentially without catastrophic interference.

The `stability-plasticity' spectrum that had been used to describe the memory retention of lookup tables as compared to the adaptability of ANNs is due to the different geometry and computational complexity of different kinds of models. Lookup tables are max-distal orthogonal and thus immune to distal and catastrophic interference. 

The weaker non-negative dissimilarity measure $ \min_{i} (|x_{i} - v_{i}|) > \delta$ was used to design a polynomial complexity model with min-distal orthogonality. The ABEL-Spline architecture was designed from first principles for desirable properties such as bounded gradients, uniform trainability, and min-distal orthogonality. Antisymmetric exponentials were used to prove the universal function approximation ability of the ABEL-Spline architecture. An implementation of the architecture was developed and evaluated experimentally. The results showed mixed performance when compared with ReLU models. The results suggest that min-distal orthogonality is too weak to substantially improve model-only continual learning, especially on uniformly sampled data that do not satisfy the conditions for min-distal orthogonality.

In summary, min-distal orthogonality is insufficient, and max-distal orthogonality is too computationally expensive for practical continual learning. Developing and investigating models with properties between the two extremes might be worthwhile. One could potentially use small or low-dimensional ($n>1$) max-distal orthogonal models as components in more elaborate multi-variable models over many variables. One need only replace $z$-Spline ANNs with sums of low-dimensional max-distal orthogonal models, for example:
$$
F(x) = g(x_{1},x_{2}) + h(x_{3},x_{4},x_{5})
$$
where $g$ and $h$ are two- or three-variable max-distal orthogonal models. One can easily extend such models to universal function approximators using antisymmetric bounded exponential layers similar to the ABEL-Spline architecture constructed for this study. Evaluating the effectiveness of such architectures is left for future work. 

It is unclear if polynomial complexity models can provide sufficient guarantees regarding their geometry to avoid distal and catastrophic interference to enable continual learning. Training augmentation or data augmentation seems necessary for practical continual learning. Based on our findings and prior work, this study conjectures that continual learning with a polynomial complexity model trained with gradient descent requires augmentation of data or training procedures.

%Here is a citation \citep{chow:68}.

% Acknowledgements and Disclosure of Funding should go at the end, before appendices and references
%All acknowledgements go at the end of the paper before appendices and references.
%Moreover, you are required to declare funding (financial activities supporting the
%submitted work) and competing interests (related financial activities outside the submitted work).
%More information about this disclosure can be found on the JMLR website.

\clearpage
\acks{Supported by the National Research Foundation of South Africa Thuthuka Grant Number 138194/0316590115. Computing resources provided by the South African Centre for High-Performance Computing (CHPC).}

% Manual new page inserted to improve the layout of sample file - not

%\newpage

\appendix

\section{Fundamental Definitions and Theorems}
\label{sec:fundamental_def_and_theorems}
This section lists standard definitions and theorems that are necessary to prove universal function approximation and can be found in analysis courses and textbooks~\citep{strichartz2000way}.  

\begin{definition}[metric space]
\label{def:matric_space}
    A metric space consists of a set $X$ and a real-valued function $d:X\times X\rightarrow\mathbb{R}$ called the \textit{metric} or \textit{distance function}, which satisfies the following:
\begin{enumerate}[itemsep=0.1mm,parsep=0.1cm]
    \item Non-negativity: $\forall x,y\in X,\ d(x,y)\geq 0$
    \item Identity of indiscernibles: $\forall x,y\in X,\ d(x,y)=0 \iff x=y$
    \item Symmetry: $\forall x,y\in X,\ d(x,y)=d(y,x)$
    \item Triangle inequality: $\forall x,y,z\in X,\ d(x,z)\leq d(x,y)+d(y,z)$
\end{enumerate}
\end{definition}

\begin{remark}
Assume the supremum norm $\lVert f \rVert_{\text{sup}} = \sup_{x \in X} |f(x)|$ for function spaces, and $d(f,g)=\sup_{x \in X} |f(x)-g(x)|$. 
\end{remark}

\begin{definition}
Let  $(X,d)$ be a metric space. Then $C_{b}(X)$ is the normed subspace of $B(X)$ comprising all continuous, bounded functions $f: X \rightarrow \mathbb{R}$.
\end{definition}

\begin{definition}
A subset $S \subseteq C_{b}(X) $ is called an \textbf{algebra} if
\begin{enumerate}
    \item it is a linear subspace of $S$, closed under addition and scalar multiplication (by arbitrary scalars $\alpha \in  \mathbb{R}$); and %\displaystyle \mathbb{R}
    \item it is closed under point-wise multiplication, for all $f,g \in S$, we have $fg \in S$.
\end{enumerate}
\end{definition}

\begin{remark}
A unital sub-algebra is an algebra of a subset $S \subseteq C_{b}(X) $ that contains the multiplicative identity or unity such as the constant function $x \mapsto 1$.
\end{remark} 

\begin{definition}
\label{def:separate_points}
We say that a set $S \subseteq C_{b}(X) $ \textbf{separates points in} $X$ if for all $x,y \in X$ with $x \neq y$, there exists $f \in S$ such that $f(x) \neq f(y)$.
\end{definition}

\begin{definition}[dense subset]
Let $(X,d)$ be a metric space. A set $Y \subseteq X$ is called \textbf{dense} in $X$ if for every $x \in X$ and every $\varepsilon>0$, there exists $y \in Y$ such that $d(x,y) < \varepsilon$.
\end{definition}

\begin{theorem}[Stone–Weierstrass]
\label{thm_stone_weierstrass}
Let $(X,d)$ be a compact metric space. Suppose that $S \subseteq C_{b}(X) $ is an algebra that separates points in $X$. Furthermore, suppose that the constant function $x \mapsto 1$ belongs to $S$. Then $S$ is dense in $C_{b}(X)$. 
\end{theorem}

\begin{remark}
    The set of ANNs is dense in the set of continuous functions. Any continuous function can be approximated to arbitrary precision using any universal function approximator. However, universal function approximation for most conventional ANNs is not proven using the Stone-Weierstrass theorem.
\end{remark}

\newpage

\section{Spline Artificial Neural Network Properties}
\label{sec:spline_additive_models}

\noindent
From Definition~\ref{def:lambda_sam}: A $z$-Spline ANN model $f(x)$, defined on $\left[0,1\right]^{n}$, with trainable parameters $\theta$ and partition number $z \in \mathbb{N}$, is a sum of $n \in \mathbb{N}$ single-variable $z$-density B-splines in each variable
\begin{equation*} %\label{eq1}
\begin{split}
f(x)  
= \sum^{n}_{j=1} f_{j}( x_{j} ) 
= \sum^{n}_{j=1} \sum_{i=1}^{4z+3}   \theta_{i,j} S_{i,j}(  x_{j} )
= \sum^{n}_{j=1} \sum_{i=1}^{4z+3}   \theta_{i,j} S((4z) x_{j} + (4 - i))
\end{split}
\end{equation*}

\begin{remark}
    $z$-Spline ANNs also inherit properties from $z$-density B-splines, such as sparsity, bounded gradients, uniform trainability, and min-distal orthogonality. Most properties are bounds that scale with $n$ with no shared parameters. Each input dimension is considered separately.
\end{remark}

%\clearpage
%\subsection{Sparsity}
\begin{proposition}[Sparsity]
\label{prop:sams_sparsity}
Let $f(x )$ be a $z$-Spline ANN, defined on the domain $x \in \left[ 0,1 \right]^{n}$, with trainable parameters $\theta \in \Theta$. Let $\lVert \nabla_{\theta} f(x ) \rVert_{0}$ denote the number of non-zero components of the gradient vector w.r.t. trainable parameters is bounded:
\begin{equation*}
\lVert \nabla_{\theta} f(x ) \rVert_{0}  
\leq 4 n
\; \forall x \in D(f)    
\end{equation*}
\end{proposition}

\begin{proof}
Let $f(x )$ be a $z$-Spline ANN, defined on the domain $x \in \left[ 0,1 \right]^{n}$, with trainable parameters $\theta \in \Theta$. Let $\lVert \nabla_{\theta} f(x ) \rVert_{0}$ denote the number of non-zero components of the gradient vector w.r.t. trainable parameters. From definition~\ref{def:lambda_sam}:

$$ \lVert \nabla_{\theta} f(x ) \rVert_{0}  
= \lVert \nabla_{\theta} \sum^{n}_{j=1} f_{j}( x_{j} ) \rVert_{0}
= \sum^{n}_{j=1} \lVert \nabla_{\theta}  f_{j}( x_{j} ) \rVert_{0} $$

From proposition~\ref{prop:z_spline_sparsity} it follows that: $ \lVert \nabla_{\theta} f(x ) \rVert_{0}  
\leq \sum^{n}_{j=1} 4 
= 4n $
\end{proof}

%\vspace{-1em}
%\subsection{Bounded Gradients}
\begin{proposition}[Bounded Parameter Gradient]
\label{prop:sams_bounded_gradient}
Let $f(x )$ be a $z$-Spline ANN, defined on the domain $x \in \left[ 0,1 \right]^{n}$, with trainable parameters $\theta \in \Theta$. For any $x  \in D(f)$, the gradient w.r.t. trainable parameters, $\nabla_{\theta} f(x )$, is bounded:
\begin{equation*} %\label{eq1}
\begin{split}
\lVert \nabla_{\theta} f(x )\rVert_{1}  
= \sum^{n}_{j=1} \sum_{i=1}^{4z+3}  \left|  S_{i,j}(  x_{j} ) \right|
< 4 n
\end{split}
\end{equation*}
\end{proposition}

\begin{proof}
    Let $f(x )$ be a $z$-Spline ANN, defined on the domain $x \in \left[ 0,1 \right]^{n}$, with trainable parameters $\theta \in \Theta$. From definition~\ref{def:lambda_sam}:

    \begin{equation*}
        \lVert \nabla_{\theta} f(x ) \rVert_{1}  
    = \lVert \nabla_{\theta} \sum^{n}_{j=1} f_{j}( x_{j} ) \rVert_{1}
    = \sum^{n}_{j=1} \lVert \nabla_{\theta}  f_{j}( x_{j} ) \rVert_{1}
    \end{equation*}

    From proposition~\ref{prop:z_spline_bounded_gradient} it follows that: $ \lVert \nabla_{\theta} f(x ) \rVert_{1} \leq \sum^{n}_{j=1} 4 = 4n $
\end{proof}

%\clearpage
%\subsection{Uniform Trainability}

\begin{proposition}[Trainability of $z$-Spline ANNs]
\label{prop:sams_trainability}
Let $f(x )$ be a $z$-Spline ANN, defined on the domain $x \in \left[ 0,1 \right]^{n}$, with trainable parameters $\theta \in \Theta$:

\begin{equation*}
    \nabla_{\theta} f( x) \neq \Vec{0}, \; \forall x \in D(f), \; 
    \forall \theta \in \Theta
\end{equation*} 
\end{proposition}

\begin{proof}
    Let $f(x )$ be a $z$-Spline ANN, defined on the domain $x \in \left[ 0,1 \right]^{n}$, with trainable parameters $\theta \in \Theta$. From definition~\ref{def:lambda_sam}:

    \begin{equation*}
         \nabla_{\theta} f(x )   
    =   \sum^{n}_{j=1} \nabla_{\theta} f_{j}( x_{j} ) 
    \end{equation*}
\noindent
From proposition~\ref{prop:z_spline_uniform_trainability}, it follows that $\nabla_{\theta} f_{j}( x_{j} )$ are non-zero. Since $f_{j}$ are independent it follows that $ \nabla_{\theta} f( x) \neq \Vec{0}$.
\end{proof}

%This makes global optimisation easy with gradient descent algorithms. The catch is that $z$-Spline ANNs are not universal function approximators. $z$-Spline ANNs can approximate multi-variable functions that are sums of single-variable functions. However, not all multi-variable functions can be decomposed into a sum of single-variable functions in each input variable.

%\subsection{Min-Distal Orthogonality}
%Min-distal orthogonality is similar to the single-variable case. However, \textbf{each} coordinate must sufficiently differ as shown in Figure~\ref{fig:sam_lesser_distal_orthogonality}, unlike lookup tables that are discussed in Section~\ref{sec:continual_learning_with_lookup_tables}, and shown in Figure~\ref{fig:lookup_table_distal_orthogonality}. 
%The caveat is that \textbf{each} coordinate must sufficiently differ to guarantee orthogonality as seen in Figure~\ref{fig:sam_lesser_distal_orthogonality}.

\begin{proposition}[min-distal orthogonality]
\label{prop:sams_lesser_distal_orthogonality}
Let $f(x)$ be a $z$-Spline ANN defined on the domain $x \in \left[ 0,1 \right]^{n}$, with partition number $z$ and trainable parameters $\theta \in \Theta$, then for any $x,y \in D(f)$:
\begin{equation*}
    \min_{i} (|x_{i} - y_{i}|) > z^{-1} \implies  \nabla_{\theta} f(x) \cdot \nabla_{\theta} f(y) = 0 
\end{equation*}
\end{proposition}

\begin{proof}
Let $f(x)$ be a $z$-Spline ANN defined on $x \in \left[ 0,1 \right]^{n}$, with partition number $z$ and trainable parameters $\theta \in \Theta$. Let $x,y \in D(f)$. From definition~\ref{def:lambda_sam}, it follows that:
\begin{equation*}
        \nabla_{\theta} f(x) \cdot \nabla_{\theta} f(y)   
   =   \bigg(\sum^{n}_{j=1} \nabla_{\theta} f_{j}( x_{j} )\bigg) 
   \cdot \bigg(\sum^{n}_{j=1} \nabla_{\theta} f_{j}( y_{j} )\bigg) 
   \end{equation*}

   Since each function $f_{j}$ is independent for each input dimension,

   \begin{equation*}
    \bigg(\sum^{n}_{j=1} \nabla_{\theta} f_{j}( x_{j} )\bigg) 
   \cdot \bigg(\sum^{n}_{j=1} \nabla_{\theta} f_{j}( y_{j} )\bigg)    
    =   \sum^{n}_{j=1} \nabla_{\theta} f_{j}( x_{j} ) 
    \cdot \nabla_{\theta} f_{j}( y_{j} )
    \end{equation*}

\noindent
Assume that $\min_{i} (|x_{j} - y_{j}|) > z^{-1}$, then $|x_{j} - y_{j}| > z^{-1}, \; \forall \, j \in \mathbb{N}$. By proposition~\ref{prop:z_spline_distal_orthogonality}, it follows that $\nabla_{\theta} f_{j}( x_{j} ) \cdot \nabla_{\theta} f_{j}( y_{j} ) = 0$, for all $j \in \mathbb{N}$. Finally, $\nabla_{\theta} f(x) \cdot \nabla_{\theta} f(y) = 0$
\end{proof}

%\noindent
%min-distal orthogonality in $z$-Spline ANNs has more substantial requirements on the differences between points than a lookup table since every input coordinate must be sufficiently different, as shown in Figure~\ref{fig:sam_lesser_distal_orthogonality}. 

\section[Antisymmetric Exponentials are Universal Function Approximators]{Antisymmetric Exponentials are Function Approximators}
\label{sec:universal_function_approximation}
This section defines and proves that antisymmetric exponentials are universal function approximators with the Stone-Weierstrass theorem (Appendix~\ref{sec:fundamental_def_and_theorems}, Theorem~\ref{thm_stone_weierstrass}).

\begin{definition}
\label{def:antisymmetric_exponentials}
Let $(X,d) \subseteq \mathbb{R}^{n}$ be a compact metric space. For all $\psi$ in the set of \textbf{antisymmetric exponentials} $ \Psi \subseteq C_{b}(X)$, there exists $M \in \mathbb{N} $, and $g_{k,j}(x_{j}), \; h_{k,j}(x_{j}) \in C(\mathbb{R})$ such that:
\begin{equation*}
\psi (x) = \sum_{k = 1}^{M} 
\Bigg( 
\exp \bigg( \Sigma_{j=1}^{n} g_{k,j}(x_{j})\bigg) 
- 
\exp\bigg( \Sigma_{j=1}^{n} h_{k,j}(x_{j})\bigg) 
\Bigg)    
\end{equation*}
\end{definition}

\begin{lemma}
\label{lemma:antisymmetric_exponentials_are_closed_under_scalar_mult}
Suppose that $\Psi \subseteq C_{b}(X) $ is the set of antisymmetric exponentials, then $\Psi$ is closed under scalar multiplication for any $\rho \in \mathbb{R}$.
\end{lemma}

\begin{proof}
Let $(X,d) \subseteq \mathbb{R}^{n}$ be a compact metric space, with $x \in (X,d)$, and components $x=(x_{1},..,x_{n})$. Let $\psi,\psi' \in \Psi$ be antisymmetric exponentials. From Definition~\ref{def:antisymmetric_exponentials}: 
\begin{equation*} 
\begin{split}
\psi 
& = \sum_{k = 1}^{M} \Bigg( \exp \bigg( \Sigma_{j=1}^{n} g_{k,j}(x_{j})\bigg) - \exp\bigg( \Sigma_{j=1}^{n} h_{k,j}(x_{j})\bigg) \Bigg)    \\
\psi'
& = \sum_{k = 1}^{M'} \Bigg( \exp\bigg( \Sigma_{j=1}^{n} g'_{k,j}(x_{j})\bigg) - \exp\bigg( \Sigma_{j=1}^{n} h'_{k,j}(x_{j})\bigg) \Bigg)   \\
\end{split}
\end{equation*}

\noindent
where $g_{k,j}(x_{j}), \; h_{k,j}(x_{j}), \; g'_{k,j}(x_{j}), \; h'_{k,j}(x_{j}) \in C(\mathbb{R})$ are continuous single-variable functions.

\noindent
Let $0<\alpha \in \mathbb{R}^{+}$. Three cases are considered $\forall \, \rho \in \mathbb{R}$ which is equal to $\alpha, -\alpha$, or zero:

\begin{enumerate}
    \item For $\psi'$ choose $g'_{k,1}(x_{1}) = g_{k,1}(x_{1}) + \log(\alpha)$, and $h'_{k,1}(x_{1}) = h_{k,1}(x_{1}) + \log(\alpha)$. Choose $g'_{k,j}(x_{j}) = g_{k,j}(x_{j})$, and $h'_{k,j}(x_{j}) = h_{k,j}(x_{j})$ for all indices $j>1 \in \mathbb{N}$ :
    $$\alpha \psi 
    = \sum_{k = 1}^{M} 
    \bigg( 
    \exp\bigg( \log(\alpha)+ \Sigma_{j=1}^{n} g_{k,j}(x_{j})\bigg) 
    - 
    \exp\bigg(  \log(\alpha)+\Sigma_{j=1}^{n} h_{k,j}(x_{j})\bigg) 
    \bigg) = \psi'$$
    \item For $\psi'$ choose $g'_{k,1}(x_{1}) = h_{k,1}(x_{1}) + \log(\alpha)$, and $h'_{k,1}(x_{1}) = g_{k,1}(x_{1}) + \log(\alpha)$. Choose $g'_{k,j}(x_{j}) = h_{k,j}(x_{j})$, and $h'_{k,j}(x_{j}) = g_{k,j}(x_{j})$ for all indices $j>1 \in \mathbb{N}$ :
    $$-\alpha \psi = 
    \sum_{k = 1}^{M} 
    \bigg( 
    \exp\bigg( \log(\alpha)+ \Sigma_{j=1}^{n} h_{k,j}(x_{j})\bigg) 
    - 
    \exp\bigg(  \log(\alpha)+\Sigma_{j=1}^{n} g_{k,j}(x_{j})\bigg) 
    \bigg) 
    = \psi'$$
    \item For $\psi'$ choose $g'_{k,j}(x_{j}) = 0$, and $h'_{k,j}(x_{j}) = 0$ for all indices $j,k \in \mathbb{N}$ :
    $$0 \psi = 0 = \sum_{k = 1}^{M'} \bigg(\exp(0) - \exp(0)
    \bigg) = \psi'$$
\end{enumerate}

\noindent
One must define single-variable interior functions as zero or absorb a constant. Since this is true for any $\rho \in \mathbb{R}$, it follows that $\Psi$ is closed under scalar multiplication.
\end{proof}

\begin{lemma}
\label{lemma:antisymmetric_exponentials_closed_under_addition}
Suppose that $\Psi \subseteq C_{b}(X) $ is the set of antisymmetric exponentials, then $\Psi$ is closed under addition.
\end{lemma}

\newpage
\begin{proof}
Let $(X,d) \subseteq \mathbb{R}^{n}$ be a compact metric space, with $x \in (X,d)$, and components $x=(x_{1},..,x_{n})$. Let $\psi,\psi',\psi'' \in \Psi$ be antisymmetric exponentials. From Definition~\ref{def:antisymmetric_exponentials}: 

\begin{equation*} 
\begin{split}
\psi 
& = \sum_{k = 1}^{M} \Bigg( \exp \bigg( \Sigma_{j=1}^{n} g_{k,j}(x_{j})\bigg) - \exp\bigg( \Sigma_{j=1}^{n} h_{k,j}(x_{j})\bigg) \Bigg)    \\
\psi'
& = \sum_{k = 1}^{M'} \Bigg( \exp\bigg( \Sigma_{j=1}^{n} g'_{k,j}(x_{j})\bigg) - \exp\bigg( \Sigma_{j=1}^{n} h'_{k,j}(x_{j})\bigg) \Bigg)   \\
\psi''
& = \sum_{k = 1}^{M''} \Bigg( \exp\bigg( \Sigma_{j=1}^{n} g''_{k,j}(x_{j})\bigg) - \exp\bigg( \Sigma_{j=1}^{n} h''_{k,j}(x_{j})\bigg) \Bigg)   \\
\end{split}
\end{equation*}

\noindent
Let $G'_{k} := \Sigma_{j=1}^{n} g'_{k,j}(x_{j}), \; H'_{k} := \Sigma_{j=1}^{n} h'_{k,j}(x_{j})$, similarly for $G_{k}, H_{k}$ and $G''_{k}, H''_{k}$. Choose the interior functions for $\psi''$ such that $G''_{k} = G_{k}, \; H''_{k} = H_{k}$ for all indices $k\leq M$. Choose $G''_{k} = G'_{k-M}, \; H''_{k} = H'_{k-M}, \; \forall k > M$ such that the indices are in range. With $M''=M+M'$, it follows that: 
\begin{equation*} 
\begin{split}
\psi +\psi'
& = \bigg( \sum_{k = 1}^{M} \bigg( \exp G_{k} - \exp H_{k} \bigg) \bigg) + \bigg( \sum_{k = 1}^{M'} \bigg( \exp G'_{k} - \exp H'_{k} \bigg) \bigg)   \\
\psi +\psi'
& = \sum_{k = 1}^{M+M'} \bigg( \exp G''_{k} - \exp H''_{k} \bigg) 
= \sum_{k = 1}^{M''} \bigg( \exp G''_{k} - \exp H''_{k} \bigg) = \psi''\\
\end{split}
\end{equation*}
\noindent
\noindent
$ \forall\, \psi,\psi' \in \Psi$, $\exists \,  \psi'' \in \Psi$ s.t. $\psi+\psi'=\psi''$, so it follows that $\Psi$ is closed under addition. 
\end{proof}

\begin{lemma}
\label{lemma:antisymmetric_exponentials_closed_under_pointwise_mult}
Suppose that $\Psi \subseteq C_{b}(X) $ is the set of antisymmetric exponentials, then $\Psi$ is closed under point-wise multiplication.
\end{lemma}

\begin{proof}
Let $(X,d) \subseteq \mathbb{R}^{n}$ be a compact metric space, with $x \in (X,d)$, and components $x=(x_{1},..,x_{n})$. Let $\psi,\psi',\psi'' \in \Psi$ be antisymmetric exponentials. From Definition~\ref{def:antisymmetric_exponentials}: 

\begin{equation*} 
\begin{split}
\psi 
& = \sum_{p = 1}^{M} \Bigg( \exp \bigg( \Sigma_{j=1}^{n} g_{p,j}(x_{j})\bigg) - \exp\bigg( \Sigma_{j=1}^{n} h_{p,j}(x_{j})\bigg) \Bigg)    \\
\psi'
& = \sum_{q = 1}^{M'} \Bigg( \exp\bigg( \Sigma_{j=1}^{n} g'_{q,j}(x_{j})\bigg) - \exp\bigg( \Sigma_{j=1}^{n} h'_{q,j}(x_{j})\bigg) \Bigg)   \\
\psi''
& = \sum_{k = 1}^{M''} \Bigg( \exp\bigg( \Sigma_{j=1}^{n} g''_{k,j}(x_{j})\bigg) - \exp\bigg( \Sigma_{j=1}^{n} h''_{k,j}(x_{j})\bigg) \Bigg)   \\
\end{split}
\end{equation*}

\noindent
Let $G'_{k} := \Sigma_{j=1}^{n} g'_{k,j}(x_{j}), \; H'_{k} := \Sigma_{j=1}^{n} h'_{k,j}(x_{j})$, similarly for $G_{k}, H_{k}$ and $G''_{k}, H''_{k}$. Multiplying $\psi$ and $\psi'$ yields $4MM'$ terms in total, with $2MM'$ positive and $2MM'$ negative exponential functions, such that $M'' = 2 M M'$. The choice of indexing is arbitrary, and $k,p,q$ are used to distinguish different functions, with $k$ being dependent on $(p,q)$. The interior functions are closed under addition and 
$G''_{k} = G_{p} +G'_{q}$ or $G''_{k} = H_{p} +H'_{q}$ for positive terms. For negative terms one has $ H''_{k} = H_{p} + G'_{q} $ or $ H''_{k} = G_{p} + H'_{q}$ such that:
\begin{small}
\begin{equation*} 
\begin{split}
\psi \psi'
& = \bigg(\sum_{p = 1}^{M} \bigg( \exp G_{p} - \exp H_{p} \bigg) \bigg) \bigg( \sum_{q = 1}^{M'} \bigg( \exp G'_{q} - \exp H'_{q} \bigg) \bigg)    \\
& = \sum_{p = 1}^{M} \sum_{q = 1}^{M'}  \bigg( \exp G_{p} - \exp H_{p} \bigg) \bigg( \exp G'_{q} - \exp H'_{q} \bigg)   \\
& = \sum_{p = 1}^{M} \sum_{q = 1}^{M'}  \bigg( \exp\bigg( G_{p} +G'_{q}\bigg) - \exp\bigg( H_{p} + G'_{q}\bigg) - \exp\bigg( G_{p} + H'_{q}\bigg) + \exp\bigg( H_{p} + H'_{q}\bigg)\bigg)    \\
& = \sum_{k = 1}^{M''}  \bigg( \exp G''_{k} - \exp H''_{k}  \bigg) = \psi''   \\
\end{split}
\end{equation*}
\end{small}

\noindent
$ \forall\, \psi,\psi' \in \Psi$, $\exists \,  \psi'' \in \Psi$ s.t. $\psi\psi'=\psi''$, so it follows that $\Psi$ is closed under multiplication. 
\end{proof}

\begin{lemma}
\label{lemma:antisymmetric_exponentials_are_algebra}
Suppose that $\Psi \subseteq C_{b}(X) $ is the set of antisymmetric exponentials, then $\Psi$ is an algebra.
\end{lemma}

\begin{proof}
Antisymmetric exponentials $\Psi \subseteq C_{b}(X) $ are closed under scalar multiplication (by lemma~\ref{lemma:antisymmetric_exponentials_are_closed_under_scalar_mult}), addition (by lemma~\ref{lemma:antisymmetric_exponentials_closed_under_addition}), and point-wise multiplication (by lemma~\ref{lemma:antisymmetric_exponentials_closed_under_pointwise_mult}). The set of antisymmetric exponentials $\Psi \subseteq C_{b}(X) $ is an algebra by definition.
\end{proof}

\begin{lemma}
\label{lemma:antisymmetric_exponentials_contains_constant}
Let $(X,d)$ be a compact metric space with $x \in X$. Suppose that $\Psi \subseteq C_{b}(X) $ is the set of antisymmetric exponentials, then the constant function $x \mapsto 1$ is an element of $\Psi$
\end{lemma}

\begin{proof}
Let $(X,d) \subseteq \mathbb{R}^{n}$ be a compact metric space, with $x \in (X,d)$, and components $x=(x_{1},..,x_{n})$. Let $\psi,\psi' \in \Psi$ be antisymmetric exponentials. From Definition~\ref{def:antisymmetric_exponentials}: 
\begin{equation*} 
\begin{split}
\psi 
& = \sum_{k = 1}^{M} \Bigg( \exp \bigg( \Sigma_{j=1}^{n} g_{k,j}(x_{j})\bigg) - \exp\bigg( \Sigma_{j=1}^{n} h_{k,j}(x_{j})\bigg) \Bigg)    \\
\end{split}
\end{equation*}

\noindent
Choose $g_{1,1}(x_{1}) = \log(2)$, and all other functions $g_{k,j}(x_{j}), \; h_{k,j}(x_{j})=0$, by substitution: 
\begin{equation*} 
\begin{split}
\psi 
& =  \exp( \log 2) - \exp(0) + \sum_{k = 2}^{M} \bigg( \exp(0) - \exp( 0) \bigg)    \\
\psi
& =  2 - 1 + \sum_{k = 2}^{M} \bigg( 1 - 1 \bigg) = 1 \\
\end{split}
\end{equation*}

The constant function $x \mapsto 1$ is in the set of antisymmetric exponentials $\Psi$.
\end{proof}

\begin{lemma}
\label{lemma:antisymmetric_exponentials_separates}
Let $(X,d)$ be a compact metric space. Suppose that $\Psi \subseteq C_{b}(X) $ is the set of antisymmetric exponentials. $\Psi$ separates points in $X$ such that or all $x,y \in X$ with $x \neq y$, there exists a $f \in \Psi$ such that $f(x) \neq f(y)$.
\end{lemma}

\begin{proof}
$\Psi$ separates points $x,y$ in $X$. Suppose $x \neq y$, without loss of generality, that the $q$th components differ: $x_{q} \neq y_{q}$. Let $\psi \in \Psi$ such that:
$$
\psi (x) = \sum_{k = 1}^{M} 
\Bigg( 
\exp \bigg( \Sigma_{j=1}^{n} g_{k,j}(x_{j})\bigg) 
- 
\exp\bigg( \Sigma_{j=1}^{n} h_{k,j}(x_{j})\bigg) 
\Bigg)       
$$

\noindent
Choose $g_{1,q}(x_{q}) = x_{q}$, and all other single-variable functions $g_{k,j}(x_{j}), \; h_{k,j}(x_{j})=0$. It follows that: $\psi(x)= \exp(x_{q}) - 1 $. Similarly $\psi (y)=\exp(y_{q}) - 1$. The exponential function is strictly monotone, so it follows:  $\psi (x) \neq \psi(y)$.
\end{proof}

\begin{theorem}
Let $(X,d)$ be a compact metric space. Suppose that $\Psi \subseteq C_{b}(X) $ is the set of antisymmetric exponentials, then $\Psi$ is dense in $C_{b}(X)$
\end{theorem}

\begin{proof}
Let $(X,d) \subset \mathbb{R}^{n}$ be a compact metric space. Suppose that $\Psi \subseteq C_{b}(X) $ is the set of antisymmetric exponentials. By lemma~\ref{lemma:antisymmetric_exponentials_are_algebra}, $\Psi$ is an algebra. By lemma~\ref{lemma:antisymmetric_exponentials_contains_constant}$, \Psi$ contains the constant function $x \mapsto 1$, and by lemma~\ref{lemma:antisymmetric_exponentials_separates}, $\Psi$ separates points in $X$. By the Stone-Weierstrass theorem~\ref{thm_stone_weierstrass}, $\Psi$ is dense in $C_{b}(X)$.
\end{proof}

\begin{remark}
    $\Psi$ is dense in the metric space of continuous bounded functions $(C_{b}(X),d)$. It follows from the definition that for any function $f \in C_{b}(X)$ and every $\varepsilon > 0$, there exists an antisymmetric exponential function $\psi \in \Psi$ such that $d(f,\psi)<\varepsilon$. 
\end{remark}

\section{ABEL-Spline Properties}
\label{sec:properties_of_abel_spline}
%\section{Definition of ABEL-Splines}
%\label{definition_abel_spline}

From Definition~\ref{def:lambda_abel_spline}:
Let $A(x)$ be an ABEL-Spline function, defined on the domain $x \in \left[ 0,1 \right]^{n}$, with trainable parameters $\theta$ and partition number $z \in \mathbb{N}$. Then, there exists $\mathcal{K} \in \mathbb{N}$, and multi-variable $z$-Spline ANN functions $F(x),G_{k}(x), H_{k}(x)$ such that:
\begin{equation*} %\label{eq1}
\begin{split}
A(x) 
:= F(x) + 
\sum_{k = 1}^{\mathcal{K}} \frac{1}{k^{2}} \bigg( \exp(G_{k}(x)) 
-  \exp(H_{k}(x)) \bigg)
\end{split}
\end{equation*}

\begin{proposition}[Sparsity]
Let $A(x )$ be a $z$-ABEL-Spline function, defined on the domain $x \in \left[ 0,1 \right]^{n}$, with trainable parameters $\theta$ and $z \geq 1$. Let $\lVert \nabla_{\theta} A(x) \rVert_{0}$ denote the number of non-zero components of the gradient vector w.r.t. trainable parameters, then for any $x  \in D(f)$, the number of non-zero components is bounded:
\begin{equation*} %\label{eq1}
\begin{split}
\lVert \nabla_{\theta} A(x ) \rVert_{0}  
\leq 4 n (2\mathcal{K}+1)
\end{split}
\end{equation*}
\end{proposition}

\begin{proof}
    Let $A(x )$ be a $z$-ABEL-Spline from Definition~\ref{def:lambda_abel_spline}. From the triangle inequality and the pseudo-norm property: $\lVert \alpha x \rVert_{0} = \lVert x \rVert_{0}, \; \forall \; x \in \mathbb{R}^{n}, \; \; \forall \; \alpha \neq 0 \in \mathbb{R}$, it follows that:
\begin{equation*} %\label{eq1}
\begin{split}
\left\| \nabla_{\theta} A(x ) \right\|_{0}  
\leq &  
\left\| \nabla_{\theta} F(x) \right\|_{0}
+ 
\sum_{k = 1}^{\mathcal{K}} 
\left\| 
\frac{1}{k^{2}} \bigg( 
\exp(G_{k}(x)) \nabla_{\theta} G_{k}(x)
 -  \exp(H_{k}(x)) \nabla_{\theta} H_{k}(x) \bigg) 
\right \|_{0}  \\
\left\| \nabla_{\theta} A(x ) \right\|_{0}  
\leq &  
\left\| \nabla_{\theta} F(x) \right\|_{0}
+ 
\sum_{k = 1}^{\mathcal{K}}  
\bigg( 
\left\| \nabla_{\theta} G_{k}(x) \right \|_{0}
 -  
\left\|  \nabla_{\theta} H_{k}(x) \right \|_{0}  \bigg) 
   \\
\end{split}
\end{equation*}

From Proposition~\ref{prop:sams_sparsity} it follows: $\left\| \nabla_{\theta} F(x) \right\|_{0} = \left\| \nabla_{\theta} G_{k}(x) \right \|_{0} = \left\|  \nabla_{\theta} H_{k}(x) \right \|_{0} = 4n$, so
\begin{equation*} %\label{eq1}
\begin{split}
\lVert \nabla_{\theta} A(x ) \rVert_{0} 
\leq &  4n  +\sum_{k = 1}^{\mathcal{K}} \big( 4n + 4n  \big) = 4 n (2\mathcal{K}+1) \\
\end{split}
\end{equation*}
The model has a total of $n(4z+3)(2\mathcal{K}+1)$ trainable parameters. The gradient vector has a maximum of $4 n (2\mathcal{K}+1)$ non-zero entries, independent of $z$. At most, the fraction of active basis functions is $\frac{4}{4z+3}$. 
\end{proof}

%\subsection{Bounded Gradients}
\begin{proposition}[Bounded gradient]
Let $A(x )$ be a $z$-ABEL-Spline function, defined on the domain $x \in \left[ 0,1 \right]^{n}$, with \textbf{bounded} trainable parameters $|\theta_{i}| < \mu, \; \forall \theta \in \Theta$. Then the gradient w.r.t. trainable parameters, $\nabla_{\theta} A(x )$, is bounded $ \forall x \in D(f)$:
\begin{equation*} %\label{eq1}
\begin{split}
\lVert \nabla_{\theta} A(x ) \rVert_{1} 
< 4 n  + \big(8 n \exp( 4 \mu n ) \big) \frac{\pi^{2}}{6}
\end{split}
\end{equation*}
\end{proposition}

\begin{proof}
    Let $A(x )$ be a $z$-ABEL-Spline. From Definition~\ref{def:lambda_abel_spline} and the triangle inequality:
\begin{equation*} %\label{eq1}
\begin{split}
\left\| \nabla_{\theta} A(x ) \right\|_{1}  
\leq &  
\lVert \nabla_{\theta}F(x) \rVert_{1}  +\sum_{k = 1}^{\mathcal{K}}  \frac{1}{k^{2}} \bigg( 
\exp(G_{k}(x)) \lVert  \nabla_{\theta} G_{k}(x)\rVert_{1} 
+
\exp(H_{k}(x)) \lVert \nabla_{\theta} H_{k}(x) \rVert_{1} 
\bigg) \\
\end{split}
\end{equation*}

From Proposition~\ref{prop:sams_bounded_gradient} it follows: $\left\| \nabla_{\theta} F(x) \right\|_{1}, \left\| \nabla_{\theta} G_{k}(x) \right \|_{1}, \left\|  \nabla_{\theta} H_{k}(x) \right \|_{1} < 4n$, so
\begin{equation*} %\label{eq1}
\begin{split}
 \lVert \nabla_{\theta} A(x ) \rVert_{1}  
< &  4 n  +\sum_{k = 1}^{\mathcal{K}}  \frac{1}{k^{2}} \bigg( 
\exp(G_{k}(x)) 4 n 
+
\exp(H_{k}(x)) 4 n 
\bigg) \\
\end{split}
\end{equation*}

Since $|\theta_{i}| < \mu, \; \forall \theta \in \Theta$, it follows that $|G_{k}|, |H_{k}| < 4 \mu n $. Consequently:
\begin{equation*} %\label{eq1}
\begin{split}
 \lVert \nabla_{\theta} A(x ) \rVert_{1}  
< &   
4 n  +\sum_{k = 1}^{\mathcal{K}}  \frac{1}{k^{2}} 4 n \bigg( 
\exp( 4 \mu n ) 
+
\exp( 4 \mu n ) 
\bigg) = 4 n  + \big(8 n \exp( 4 \mu n ) \big) \sum_{k = 1}^{\mathcal{K}}  \frac{1}{k^{2}} \\
\end{split}
\end{equation*}

By substitution, the absolutely convergent series $\sum_{k} k^{-2}$ gives:
\begin{equation*} %\label{eq1}
\begin{split}
\lVert \nabla_{\theta} A(x ) \rVert_{1}  
< &  4 n  + \big(8 n \exp( 4 \mu n ) \big) \sum_{k = 1}^{\infty}  \frac{1}{k^{2}} = 4 n  + \big(8 n \exp( 4 \mu n ) \big) \frac{\pi^{2}}{6}\\
\end{split}
\end{equation*}
Thus,$\lVert \nabla_{\theta} A(x ) \rVert_{1}$ is bounded $\forall x \in [0,1]^{n}$ and bounded parameters $|\theta_{i}| < \mu, \; \forall \theta \in \Theta$
\end{proof}

\begin{remark}
    The factor of $k^{-2}$ inside the expression for ABEL-Spline is necessary to ensure the sum converges in the limit of infinitely many exponential terms $\mathcal{K} \to \infty$. This technique could be used on other ANN models to stabilise training.
\end{remark}

%\subsection{Uniform Trainability}
\begin{proposition}[Trainability of ABEL-Spline]
Let $A(x )$ be a $z$-ABEL-Spline function, defined on the domain $x \in \left[ 0,1 \right]^{n}$, with trainable parameters $\theta \in \Theta$. For any $x  \in D(A)$, the gradient w.r.t. trainable parameters, $\nabla_{\theta} A(x )$, is non-zero:
\begin{equation*}
    \nabla_{\theta} A(\theta, x) \neq \Vec{0}
\end{equation*}
\end{proposition}

\begin{proof}
    Let $A(x )$ be a $z$-ABEL-Spline function. From the Definition~\ref{def:lambda_abel_spline}, it follows that:
\begin{equation*} %\label{eq1}
\begin{split}
\nabla_{\theta} A(x) 
&= \nabla_{\theta} 
\bigg( F(x) + 
\sum_{k = 1}^{\mathcal{K}} 
\frac{1}{k^{2}} \bigg( 
\exp(G_{k}(x)) 
 -  \exp(H_{k}(x)) \bigg) 
 \bigg) \\
\nabla_{\theta} A(x) 
&= \nabla_{\theta} 
 F(x) + 
\sum_{k = 1}^{\mathcal{K}} 
\frac{1}{k^{2}} \bigg( 
\exp(G_{k}(x)) \nabla_{\theta} G_{k}(x)
 -  \exp(H_{k}(x)) \nabla_{\theta} H_{k}(x) \bigg)  \\
\end{split}
\end{equation*}

From Proposition~\ref{prop:sams_trainability} it follows that:
\begin{equation*} %\label{eq1}
\begin{split}
\nabla_{\theta} 
 F(x) \neq \Vec{0}\; , \;
 \nabla_{\theta} 
 G_{k}(x) \neq \Vec{0} \; , \;
 \nabla_{\theta} 
 H_{k}(x) \neq \Vec{0}
\end{split}
\end{equation*}

The functions $F, G_{k}, H_{k}$ are independently parameterised, so $\nabla_{\theta} F, \nabla_{\theta}G_{k}, \nabla_{\theta} H_{k}$ must be linearly independent. Any linear combination of linearly independent vectors with non-zero coefficients ($\exp(z)\neq 0, \; \forall z \in \mathbb{R}$) must be non-zero.
\end{proof}

%\newpage
%\subsection{Min-Distal Orthogonality}

\begin{proposition}[min-distal orthogonal ABEL-Spline]
Let $A(x)$ be an ABEL-Spline function $A(x)$, defined on the domain $x \in \left[ 0,1 \right]^{n}$, with partition number $z$ and trainable parameters $\theta \in \Theta$, then for any $x,y \in D(f)$:
\begin{equation*}
    \min_{i} (|x_{i} - y_{i}|) > z^{-1} \implies  \nabla_{\theta} A(x) \cdot \nabla_{\theta} A(y) = 0 
\end{equation*}
\end{proposition}

\begin{proof}
    Let $A(x )$ be a $z$-ABEL-Spline function. From the Definition~\ref{def:lambda_abel_spline}, it follows that:
\begin{equation*} %\label{eq1}
\begin{split}
\nabla_{\theta} A(x) 
&= \nabla_{\theta} 
 F(x) + 
\sum_{k = 1}^{\mathcal{K}} 
\frac{1}{k^{2}} \bigg( 
\exp(G_{k}(x)) \nabla_{\theta} G_{k}(x)
 -  \exp(H_{k}(x)) \nabla_{\theta} H_{k}(x) \bigg)  \\
\end{split}
\end{equation*}

Assume that $ \min_{i} (|x_{i} - y_{i}|) > z^{-1}$. It follows from proposition~\ref{prop:sams_lesser_distal_orthogonality}:
\begin{equation*}
\nabla_{\theta} \Phi(x) 
\cdot
 \nabla_{\theta} \Phi(y) = 0,  
 \: \forall \; \Phi \in  \{ F,G_{1},.., G_{k}, H_{1},.., H_{k}  \} \\
\end{equation*}

The functions are independently parameterised, so $ \forall x, y \in [0,1]^{n}$ the cross terms are:

\begin{equation*}
\nabla_{\theta} \Psi(x) 
\cdot
 \nabla_{\theta} \Phi(y) = 0,  
 \: \forall \; \Psi \neq \Phi \in  \{ F,G_{1},.., G_{k}, H_{1},.., H_{k}  \} \\
\end{equation*}

Substituting $\nabla_{\theta} \Phi(x) \cdot \nabla_{\theta} \Phi(y) = 0 $, and not counting zeroed cross-terms it follows that:  
\begin{equation*} %\label{eq1}
\begin{split}
\nabla_{\theta} A(x) 
        \cdot 
        \nabla_{\theta} A(y) 
&= 0 + 
\sum_{k = 1}^{\mathcal{K}} 
\frac{1}{k^{4}} \bigg( 
\exp(G_{k}(x)+G_{k}(y)) \, 0
 -  \exp(H_{k}(x)+H_{k}(y)) \, 0 \bigg)  \\
\end{split}
\end{equation*}

Finally, $ \nabla_{\theta} A(x) \cdot \nabla_{\theta} A(y) = 0$, so ABEL-Splines are min-distal orthogonal models.
\end{proof}

\vskip 0.2in
\bibliography{main}

\end{document}